\documentclass[ejs]{imsart}

\RequirePackage[OT1]{fontenc}
\RequirePackage{amsthm,amsmath,amssymb,subcaption,graphicx,epstopdf}
\RequirePackage[numbers]{natbib}
\RequirePackage[colorlinks,citecolor=blue,urlcolor=blue]{hyperref}

\pubyear{2017}
\volume{11}
\firstpage{50}
\lastpage{77}

\startlocaldefs
\numberwithin{equation}{section}
\theoremstyle{plain}
\newtheorem{theorem}{Theorem}[section]

\theoremstyle{definition}
\newtheorem{remark}{Remark}	

\theoremstyle{definition}
\newtheorem{assump}{Assumption}

\theoremstyle{definition}
\newtheorem{lemma}{Lemma}

\theoremstyle{definition}
\newtheorem{corollary}{Corollary}

\theoremstyle{definition}

\endlocaldefs

\usepackage{mathtools}

\DeclarePairedDelimiter\floor{\lfloor}{\rfloor}
\def\tr{\text{tr}}
\def\rank{\text{rank}}
\def\eps{\varepsilon}
\def\Var{\textrm{Var}}
\begin{document}

\begin{frontmatter}
\title{Estimation of low rank density matrices by Pauli measurements}
\runtitle{Low Rank Density Matrix Estimation}
\runauthor{Dong Xia}

\begin{aug}
\author{\fnms{Dong Xia} \ead[label=e1]{dongxia@stat.wisc.edu}}

\address{1300 University Ave, Madison, WI 53706, USA.\\
\printead{e1}}


\affiliation{Georgia Institute of Technology}
\end{aug}

\begin{abstract}
Density matrices are positively semi-definite Hermitian matrices with unit trace that describe the states of quantum systems. Many quantum systems of physical interest can be represented as high-dimensional low rank density matrices. A popular problem in {\it quantum state tomography} (QST) is to estimate the unknown low rank density matrix of a quantum system by conducting Pauli measurements. Our main contribution is twofold. First, we establish the minimax lower bounds in Schatten $p$-norms with $1\leq p\leq +\infty$ for low rank density matrices estimation by Pauli measurements. In our previous paper \cite{koltchinskii15optimal}, these minimax lower bounds are proved under the trace regression model with Gaussian noise and the noise is assumed to have common variance. In this paper, we prove these bounds under the Binomial observation model which meets the actual model in QST.

Second, we study the Dantzig estimator (DE) for estimating the unknown low rank density matrix under the Binomial observation model by using Pauli measurements. In our previous papers \cite{koltchinskii15optimal} and \cite{Xia2016Estimation}, we studied the least squares estimator and the projection estimator, where we proved the optimal convergence rates for the least squares estimator in Schatten $p$-norms with $1\leq p\leq 2$ and, under a stronger condition, the optimal convergence rates for the projection estimator in Schatten $p$-norms with $1\leq p\leq +\infty$. In this paper, we show that the results of these two distinct estimators can be simultaneously obtained by the Dantzig estimator. Moreover, better convergence rates in Schatten norm distances can be proved for Dantzig estimator under conditions weaker than those needed in \cite{koltchinskii15optimal} and \cite{Xia2016Estimation}.
When the objective function of DE is replaced by the negative von Neumann entropy, we obtain sharp convergence rate in Kullback-Leibler divergence. 
\end{abstract}

\begin{keyword}[class=MSC]
\kwd[Primary ]{62J99,81P50}
\kwd[; secondary ]{62H12}
\end{keyword}

\begin{keyword}
\kwd{Quantum state tomography}
\kwd{low rank}
\kwd{Schatten $p$-norms}
\kwd{density matrix}
\end{keyword}

\end{frontmatter}

\section{Introduction}
\label{intro}
Let ${\mathbb H}_m$ be the set of all Hermitian matrices: ${\mathbb H}_m:=\{A\in {\mathbb C}^{m\times m}:A=A^{\ast}\}$ with $A^{\ast}$ denoting the adjoint matrix of $A.$
For $A\in\mathbb{H}_m,$ $\text{tr}(A)$ denotes the trace of $A$ and $A\succcurlyeq 0$ means that $A$ is positively semi-definite (i.e., all its eigenvalues are nonnegative). Let $\mathcal{S}_m:=\left\{S\in\mathbb{H}_m: 
S\succcurlyeq 0, \textrm{tr}(S)=1\right\}$ be the set of all positively semi-definite $m\times m$ Hermitian matrices of unit trace which are called {\it density matrices}.  

In quantum mechanics, the state of a quantum system is often characterized (or at least approximated) by a density matrix $\rho\in\mathcal{S}_m$. The goal of {\it quantum state tomography} (QST) is to estimate
the unknown state $\rho$ based on a number of measurements conducted on the systems prepared in state $\rho$ 
(see \cite{gross2010quantum}, \cite{gross2011recovering}, \cite{koltchinskii2011neumann}, \cite{caioptimal} and references therein). The difficulty of QST is that the dimension
$m$ grows exponentially as the system size increases. For instance, for a quantum system consisting of $b$ qubits, its density matrix $\rho\in\mathcal{S}_m$ with $m=2^b$.
Fortunately, density matrices of many important quantum states (for instance, pure states) are of low rank which can significantly reduce the complexity of the estimation problem.
In this paper, we focus on the following set of low rank density matrices,
\begin{equation}\label{bpd}
 \mathcal{S}_{m,r}:=\{S\in\mathcal{S}_m: \rank(S)\leq r\}.
\end{equation}
Typically, the statistical model of QST is as follows.
Given an observable $A\in {\mathbb H}_m$ (in this paper, $A$ represents a Pauli matrix) with spectral representation $A=\sum_{j=1}^{m'}\lambda_j P_j,$ where $m'\leq m,$ $\lambda_j$ being the distinct eigenvalues of $A$ and $P_j$
being the corresponding mutually orthogonal eigenprojectors, the outcome of a measurement 
of $A$ on the system prepared in state $\rho$ is a random variable $O$ taking values $\lambda_j$
with probability $\textrm{tr}(\rho P_j).$ Put it differently, $\mathbb{P}(O=\lambda_j)=\langle\rho,P_j \rangle, j=1,2,\ldots,m$.
Then, it is easy to check $\mathbb{E}_{\rho}O=\tr(\rho A)$ and the variance of the outcome $O$ depends on both $\rho$ and $A$.
Usually, given an observable $A$, multiple measurements of $A$ are performed on independently and identically prepared quantum systems and the average outcome $Y$ is taken as the output, whose variance can be significantly smaller than the
outcome $O$ from a single measurement. For instance, given the Observable $A$, it is used to conduct measurement on $K$ independently and identically prepared quantum systems, producing the outcomes $O_1,\ldots,O_K$. Then $Y$ is taken as $Y:=K^{-1}\sum_{k=1}^K O_k$. Typically, there are many possible choices for the Observable $A$.
 A common approach is to choose an observable $A$ at random,
assuming that it is the value of a random variable $X$ with some design distribution 
$\Pi$ in a subset of ${\mathbb H}_m.$ More precisely, given a sample of $n$ i.i.d. copies 
$X_1,\dots, X_n$ of $X$, (multiple) measurements are performed for each of them on quantum systems identically
prepared in state $\rho$ resulting into the (average) outcomes $Y_1,\dots, Y_n.$ Based on 
the data $(X_1,Y_1), \dots, (X_n,Y_n),$ the goal is to estimate the underlying density matrix 
$\rho\in\mathcal{S}_m$. Clearly, the observations satisfy the following model
\footnote{It worths to point out that there is a trace regression model with bounded response studied in \cite{koltchinskii15optimal} where $Y_i$ represents the outcome of measuring $X_i$ on a single quantum system without repetition. That is, the data given under this model is $(X_1,O_1),\ldots,(X_n,O_n)$.
The purpose of considering this model is to study how many quantum systems should be produced for estimating the underlying density matrix. Here, model (\ref{trace_regression}) is more related to compressed sensing and primarily focuses on the question: what is the smallest number of Pauli measurements needed to reconstruct the low rank density matrix.}
\begin{equation}
\label{trace_regression}
Y_j={\rm tr}(\rho X_j)+ \xi_j,\ j=1,\dots, n, 
\end{equation}
where $\xi_j$ is a random noise satisfying 
the condition 
$$
{\mathbb E}_{\rho}(\xi_j|X_j)=0\quad \textrm{and}\quad \Var_{\rho}(\xi_j|X_j)=\frac{V(\rho,X_j)}{K}
$$
 for $j=1,\dots, n$ where $V(\rho,X_j)$ denotes the corresponding variance which depends on the measurement $X_j$ and the density matrix $\rho$. This is a special 
case of the so called {\it trace regression model} in the recent literature
(see, e.g., \cite{koltchinskii2011oracle}, \cite{koltchinskii2011nuclear}, \cite{negahban2012restricted}, \cite{rohde2011estimation}, \cite{wang2013asymptotic} and references therein). By CLT, $\big(\sqrt{K}\xi_j|X_j\big)$ can be approximated by a centered Gaussian random variable with variance $V(\rho,X_j)$ as long as $K$ is large enough.

Low rank density matrices ({\it quantum state tomography}) have been studied intensively both in the quantum physics community and in the statistical learning community. In Gross~\cite{gross2011recovering}, Gross et al.~\cite{gross2010quantum} and Liu~\cite{liu2011universal}, the authors introduced the techniques used for matrix compressed sensing problems into {\it quantum state tomography} in order to estimate an unknown low rank state of quantum systems. As a result, the matrix LASSO estimator (see \cite{liu2011universal}, \cite{negahban2012restricted} and \cite{rohde2011estimation}) and matrix Dantzig estimator (\cite{candes2011tight}, \cite{flammia2012quantum} and \cite{liu2011universal}) can be immediately applied in the settings of low rank density matrices.
The {\it restricted isometry property} (RIP) is the key technical tool needed in proving the consistency of these estimators, which requires $n\gtrsim mr\log^6(m)$ Pauli measurements (see Section~\ref{paulisec}). In addition to these methods based on convex programming, a rank penalized estimator was studied in \cite{alquier2013rank} through a rank penalization on a linear estimator reconstructed based on all the nontrivial Pauli measurements. A least squares estimator with penalization by von Neumann entropy was studied in \cite{koltchinskii2011neumann} where an upper bound for the Kullback-Leibler divergence was also proved.

Although those papers mentioned above have provided many meaningful results on low rank density matrices estimation, some other important problems are partially resolved just recently.
The first problem is related to the statistical lower bounds in Kullback-Leibler divergence and Schatten norm distances for low rank density matrices estimation. In \cite{flammia2012quantum}, the authors proved a minimax lower bound for the Schatten $1$-norm (trace norm) distance. Then a method is developed in \cite{koltchinskii15optimal} proving the minimax lower bounds in both Kullback-Leibler divergence and all Schatten norm distances. Those bounds are established under the trace regression model with Gaussian noise (\ref{trace_regression}) and the trace regression model with bounded response. Moreover, a least squares estimator penalized by von Neumann entropy was also studied in \cite{koltchinskii15optimal}, achieving the optimal convergence rates in Schatten $p$-norms for $1\leq p\leq 2$ which match the minimax lower bounds. Similar optimality has also been shown for Kullback-Leibler divergence. Even though the minimax lower bounds under Gaussian noise model has been developed in \cite{koltchinskii15optimal}, there are two unresolved questions. One is that the practical model for each outcome is actually a multinomial model (Binomial model in the case of Pauli measurements), rather than a Gaussian noise model unless the problem is studied in the asymptotical settings. Another question is that the minimax lower bounds developed in \cite{koltchinskii15optimal} are based on the assumption that each $\xi_j$ has the same variance, which is certainly not true in QST. These questions will be resolved in this paper and minimax lower bounds will be developed under the Binomial observation model (the typical model for Pauli measurements) where the noise variance depends on both the underlying density matrix $\rho$ and the corresponding Pauli measurement $X_i$.

Another problem is related to estimators achieving the upper bounds in Schatten $p$-norms for all $1\leq p\leq +\infty$. In the trace regression model with Gaussian noise, the sharp upper bounds in Schatten $p$-norms have been proved in \cite{koltchinskii15optimal} for a least squares estimator (denoted by $\tilde{\rho}$ in the following context). In particular, these upper bounds (except the logarithmic terms) in Schatten norms are given as follows (assuming that $V(\rho,X_j)\leq \sigma_{\xi}^2$ for all $X_j$):
\begin{equation}\label{schattenPkolt1}
\|\tilde{\rho}-\rho\|_p\leq C\bigg(\frac{\sigma_{\xi}m^{3/2}r^{1/p}}{\sqrt{n}}\bigwedge \Big(\frac{\sigma_{\xi}m^{3/2}}{\sqrt{n}}\Big)^{1-\frac{1}{p}}\bigg)\bigwedge 2,\quad 1\leq p\leq 2
\end{equation}
for some constant $C>0$. These bounds hold as long as $\sigma_{\xi}\gtrsim \frac{1}{\sqrt{nm}}$ for Pauli measurements. It is interesting to notice that the second term in (\ref{schattenPkolt1}) implies that $\frac{m}{n}$(logarithmic factors)$\to 0$ is enough to guarantee $\|\tilde{\rho}-\rho\|_2\to 0$ as $m,n\to\infty$. The convergence rates in other Schatten norm distances are proved in
 a recent work \cite{Xia2016Estimation} where a projection estimator is considered and upper bounds in the form (\ref{schattenPkolt1}) are obtained for all $1\leq p\leq +\infty$.  The bounds established in \cite{Xia2016Estimation} relies crucially on the assumption that $\sigma_{\xi}\geq \frac{1}{m}$ for Pauli measurements (if $\sigma_{\xi}\leq \frac{1}{m}$, the bound holds by replacing $\sigma_{\xi}$ with $\frac{1}{m}$). Clearly, one question is that whether the performances of the least squares estimator and the projection estimator can be simultaneously obtained by a single estimator.  And another question is what bounds can we get when $\sigma_{\xi}$ is small?
We seek to answer this question by considering a Dantizg-type estimator for Pauli measurements. Its Schatten $p$-norm convergence rates have the same form as (\ref{schattenPkolt1}) for $1\leq p\leq 2$ as long as $\sigma_{\xi}\gtrsim \frac{1}{\sqrt{nm}}$. In addition, when $\sigma_{\xi}\gtrsim \frac{1}{m}$, its Schatten $p$-norm convergence rates can be obtained for all $1\leq p\leq +\infty$. In other words, the Dantzig-type estimator achieves the performances as both the least squares estimator in \cite{koltchinskii15optimal} and the projection estimator in \cite{Xia2016Estimation}, under the same conditions.
Another advantage is that nontrivial convergence rates in Schatten $p$-norms for $1\leq p\leq +\infty$ can also be obtained under weaker conditions on $\sigma_{\xi}$. A summary of these results can be found in Section~\ref{discusssec}.

\section{Backgrounds and preliminaries}

\subsection{Notations}
The notation $\langle \cdot,\cdot \rangle$ is used 
for both the Euclidean inner product in ${\mathbb C}^m$ and for the Hilbert--Schmidt inner 
product in ${\mathbb H}_m.$
$\|\cdot\|_p, p\geq 1$ will be used to denote the 
Schatten $p$-norm in $\mathbb{H}_m,$ namely $\|A\|_p^p=
\sum\limits_{j=1}^m |\lambda_j(A)|^p,\ A\in\mathbb{H}_m,$ $\lambda_1(A)\geq \ldots \geq \lambda_m(A)$ being the eigenvalues of $A.$ In particular,  $\|\cdot\|_2$ denotes the Hilbert--Schmidt (or Frobenius) norm,
 $\|\cdot\|_1$ denotes the nuclear (or trace) norm and
$\|\cdot\|_{\infty}$ denotes the operator (or spectral) norm: $\|A\|_{\infty}=\max_{1\leq j\leq m}|\lambda_j(A)|.$

Given $A\in {\mathbb H}_m,$ and a random variable $X$ in 
${\mathbb H}_m$ with a distribution $\Pi$, we write $\|A\|_{L_2(\Pi)}^2=
\int_{{\mathbb H}_m}\langle A,x\rangle^2\Pi(dx)$. Similarly, define 
$$
\langle A,B\rangle_{L_2(\Pi)}=\int_{\mathbb{H}_m}\langle A,x\rangle \langle B,x\rangle\Pi(dx).
$$
In what 
follows, $\Pi$ will be typically the uniform distribution in an orthonormal basis 
${\mathcal E}=\{E_1,\dots, E_{m^2}\}\subset {\mathbb H}_m,$ implying that 
$$
\|A\|_{L_2(\Pi)}^2 = m^{-2}\|A\|_2^2\quad\text{and}\quad \langle A,B\rangle_{L_2(\Pi)}=\frac{1}{m^2}\langle A,B\rangle, 
$$
so, the $L_2(\Pi)$-norm is just a rescaled Hilbert--Schmidt norm.
In addition, let $\Pi_n$ denote the empirical distribution based on the sample $(X_1,\ldots,X_n)$ such that $\|A\|_{L_2(\Pi_n)}^2=\frac{1}{n}\sum_{i=1}^n\big<A,X_i\big>^2$.

The non-commutative Kullback-Leibler divergence (or relative entropy distance) $K(\cdot\|\cdot)$ is defined as (see also \cite{Nielsen2000}):
\begin{equation*}
 K(S_1\|S_2):=\big<S_1,\log S_1-\log S_2\bigr>
\end{equation*}
for $S_1,S_2\in\mathcal{S}_m$.
If $\log S_2$ is not well-defined (for instance, some of the eigenvalues of $S_2$ are equal to $0$), we set $K(S_1\|S_2)=+\infty$. The symmetrized version of Kullback-Leibler divergence is defined as
\begin{equation*}
 K(S_1;S_2):= K(S_1\|S_2)+K(S_2\|S_1)=\bigl<S_1-S_2,\log S_1-\log S_2\bigr>.
\end{equation*}

$C, C_1, C',c, c',$ etc will denote constants (that do not depend on parameters of interest)
 whose values could change from line to line (or, even, within the same line)
without further notice.  For nonnegative $A$ and $B,$ $A\lesssim B$ (equivalently, $B\gtrsim A$) means that $A\leq CB$ for some absolute constant $C>0,$ and $A\asymp B$ means that 
$A\lesssim B$ and $B\lesssim A$ simultaneously.  Moreover, by writing $A\lesssim_{\log(m,n,K)}B$, we mean that $A\leq C_1B\log^{c_1}m\log^{c_2}n\log^{c_3}K$ for some absolute constants $C_1,c_1,c_2,c_3$. 

\subsection{Sampling from Pauli basis and Binomial observation model}\label{paulisec}
The spin-$\frac{1}{2}$ particle is the simplest example of a two-state quantum system, which is conventionally called a qubit. The state of a single qubit is determined by its spin: up, down or
a superposition of both. One most popular Observable for a single qubit system is usually represented by Pauli matrices. They are given as
\begin{equation}
 \sigma_0:=\left(\begin{array}{cc}1&0\\0&1 \end{array}\right),\quad \sigma_x:=\left(\begin{array}{cc}0&1\\1&0 \end{array}\right),
  \quad \sigma_y:=\left(\begin{array}{cc}0&i\\-i&0 \end{array}\right),\quad \sigma_z:=\left(\begin{array}{cc}1&0\\0&-1 \end{array}\right).
\end{equation}
The matrices $\sigma_x,\sigma_y,\sigma_z$ correspond to the spin along the coordinate axes in $\mathbb{R}^3$. The additional matrix $\sigma_0$ represents a trivial operation on the single
qubit system. Define $W_{\alpha}:=\frac{1}{\sqrt{2}}\sigma_{\alpha}$
for $\alpha=0,x,y,z$. Then $\{W_0,W_x,W_y,W_z\}$ consists of an orthonormal basis of $\mathbb{H}_2$ which is conventionally called the Pauli basis. The Pauli matrices can be easily generalized for multi-qubit systems. Indeed,
for a system with $b$ qubits, the normalized Pauli matrices are constructed as 
\begin{equation}\label{mcalE}
\mathcal{E}:=\big\{W_{\alpha_1}\otimes W_{\alpha_2}\otimes\ldots\otimes W_{\alpha_b}: \alpha_1,\ldots,\alpha_b\in\{0,x,y,z\} \big\},
\end{equation}
which actually form an orthonormal basis of $\mathbb{H}_m$ with $m=2^b$. 
We rearrange the Pauli matrices $\mathcal{E}:=\{E_i, 1\leq i\leq m^2\}$ such that $E_1=W_0\otimes\ldots\otimes W_0=\frac{1}{\sqrt{m}}I_m$ and $E_2,\ldots,E_{m^2}$ denote the rest of Pauli matrices in (\ref{mcalE}).
An obvious fact is that $\frac{1}{\sqrt{m}}$ is the only eigenvalue of $E_1$ and $\pm\frac{1}{\sqrt{m}}$ are the eigenvalues of $E_2,\ldots,E_{m^2}$ with the same multiplicity such
that $\tr(E_k)=0$ for $2\leq k\leq m^2$. In other words, one has the spectral decomposition 
$E_k=\frac{1}{\sqrt{m}}P_k^{+}-\frac{1}{\sqrt{m}}P_k^-$ with $\rank(P_k^+)=\rank(P_k^-)=\frac{m}{2}$ for $2\leq k\leq m^2$. Note that $E_1=\frac{1}{\sqrt{m}}I_m$ with $P_1^-=0$. By measuring $E_k$
on a $b$-qubits system prepared in the state $\rho\in\mathcal{S}_m$ with $m=2^b$, the outcome $\tau_k$ is a random variable taking values $\pm\frac{1}{\sqrt{m}}$ with probability $\langle\rho,P_k^{\pm}\rangle$ 
and $\mathbb{E}_{\rho}\tau_k=\langle\rho,E_k\rangle$ for $1\leq k\leq m^2$. If we represent $\rho$ in the Pauli basis
$$
\rho=\sum_{k=1}^{m^2}\frac{\alpha_k}{\sqrt{m}}E_k
$$
with $\alpha_1=1$ and $|\alpha_k|\leq 1$ for $2\leq k\leq m^2$. Then $\mathbb{P}_{\rho}(\tau_k=\frac{\pm 1}{\sqrt{m}})=\frac{1\pm \alpha_k}{2}$ and $\text{Var}_{\rho}(\tau_k)=\frac{1-\alpha_k^2}{m}$ depending only on the coefficient $\alpha_k$.
Since $\|\rho\|_2^2=\sum_{k=1}^{m^2}\frac{\alpha_k^2}{m}\leq 1$, it indicates that $\text{Card}\big(\{k: |\alpha_k|>\frac{1}{2}\}\big)\leq 4m$. Therefore, for most of
$k$ (at least $m^2-4m$), $\text{Var}_{\rho}(\tau_k)\geq \frac{1}{2m}$.

A standard approach in QST is to randomly select $X$ uniformly from the Pauli basis $\mathcal{E}$. Then multiple measurements of $X$ are conducted on quantum systems independently
prepared in the same (unknown) state $\rho$. Suppose that $K$ measurements are performed, resulting into the outcomes $O_1,\ldots,O_K$. 
Clearly, $|O_k|=\frac{1}{\sqrt{m}}$ for $1\leq k\leq K$. Let $X_1,X_2,\ldots,X_n$ be i.i.d. random Pauli matrices sampled uniformly from $\mathcal{E}$ with replacement. For each $X_i$, $K$ measurements are conducted and the outcomes are collected. Let $K_i^+$ denote the number of outcomes $+\frac{1}{\sqrt{m}}$ and $K_i^-$ denote the number of outcomes  $-\frac{1}{\sqrt{m}}$. Then $K_i^++K_i^-=K$. It is clear that $K_i^+$ has a Binomial distribution and $K_i^+\sim\textrm{Bin}\big(K,\langle\rho, X_i^+\rangle\big)$ where $X_i^+$ and $X_i^-$ represent the spectral projectors of $X_i$ corresponding to the eigenvalues $+\frac{1}{\sqrt{m}}$ and $-\frac{1}{\sqrt{m}}$. Moreover, $K^{-1}\sum_{k=1}^K O_k=\frac{K_i^+-K_i^-}{K\sqrt{m}}$. In other words, if we define $Y_i=\frac{K_i^+-K_i^-}{K\sqrt{m}}$, then
$(X_i,Y_i)$ satisfies the trace regression model (\ref{trace_regression}) $Y_i=\langle\rho,X_i\rangle+\xi_i$ with $\mathbb{E}_{\rho}(\xi_i|X_i)=0$ and $\mathbb{E}_{\rho}(\xi_i^2|X_i)\leq \frac{1}{Km}$. By assuming $\xi_i|X_i\sim \mathcal{N}(0,\sigma_{\xi}^2)$ for all $1\leq i\leq n$, the minimax lower bounds of estimating low rank $\rho\in \mathcal{S}_{m,r}$ based on the data $(X_1,Y_1),\ldots,(X_n,Y_n)$ was established in \cite{koltchinskii15optimal}. Clearly, the assumption of Gaussian noise with common variance is not true in practice. In section~\ref{minimaxsec}, we prove the minimax lower bounds based on the data $(X_1,K_1^+),\ldots,(X_n,K_n^+)$ where $\big(K_i^+\mid X_i\big)$ has a Binomial distribution for $1\leq i\leq n$. Formally, our model is described as follows.
\begin{assump}[Binomial observation model]
\label{trassump}
Let $\mathcal{E}$ be the Pauli basis as in (\ref{mcalE}).
Let $X$ be sampled uniformly from $\mathcal{E}$, then measurements of $X$ are conducted on $K$ quantum systems independently and identically prepared in the state $\rho$ . Let $K^+$ be the number of $+\frac{1}{\sqrt{m}}$ collected from the $K$ random outcomes.
The data $(X_1,K_1^+),\dots (X_n,K_n^+)$ consists of $n$ i.i.d.
copies of $(X,K^+).$ Let $\mathbb{P}_{\rho}$ denote the probability distribution of $(X_1,K_1^+),\dots (X_n,K_n^+)$.
\end{assump}

\subsection{Some useful lemmas}
The following well known {\it interpolation inequality} for 
Schatten $p$-norms will be used to extend the bounds proved for $p=1$ and $p=\infty$
to the whole range of its values. It easily follows from similar bounds in $\ell_p$-spaces. 

\begin{lemma}[Interpolation inequality]
\label{interlem}
 For $1\leq p<q<r\leq\infty$, and let $\mu\in[0,1]$ be such that
\begin{equation*}
 \frac{\mu}{p}+\frac{1-\mu}{r}=\frac{1}{q}.
\end{equation*}
Then, for all $A\in\mathbb{H}_m,$
\begin{equation*}
 \|A\|_q\leq \|A\|_p^{\mu}\|A\|_r^{1-\mu}.
\end{equation*}
\end{lemma}

Given a subspace $L\subset {\mathbb C}^m,$ 
$L^{\perp}$ denotes the orthogonal complement of $L$ and  $P_L$ denotes the orthogonal projection onto $L.$ Let 
${\mathcal P}_L, {\mathcal P}_L^{\perp}$ be orthogonal projection operators 
in the space ${\mathbb H}_m$ (equipped with the Hilbert--Schmidt inner product),
defined as follows:
$$
{\mathcal P}_L^{\perp}(A)=P_{L^{\perp}}AP_{L^{\perp}},\ \ {\mathcal P}_L(A)=A-P_{L^{\perp}}AP_{L^{\perp}}.
$$
These two operators split any Hermitian matrix $A$ into two orthogonal parts, 
${\mathcal P}_L(A)$ and ${\mathcal P}_L^{\perp}(A),$ the first one being of rank
at most $2{\rm dim}(L).$

Non-commutative (matrix) versions of Bernstein inequality will be used in what 
follows. Lemma~\ref{matBernlem} is an unbounded version of the standard matrix Bernstein inequality (see \cite{tropp2012user}) whose proof is provided in \cite{koltchinskii2011neumann}. Recall that, for any $\alpha\geq 1$, the $\psi_{\alpha}$-norm
of a random variable $Z$ is defined as
$$
\|Z\|_{\psi_{\alpha}}:=\inf\Big\{C>0: \mathbb{E}e^{|Z|^{\alpha}/C^{\alpha}}\leq 2\Big\}.
$$ 

\begin{lemma}
\label{matBernlem}
Let $X,X_1,\ldots,X_n\in\mathbb{H}_m$ be i.i.d. random matrices with $\mathbb{E}X=0,$ $\sigma_X^2:=\|\mathbb{E}X^2\|_{\infty}$ and $U_X^{(\alpha)}\geq \max\big(\|\|X\|_{\infty}\|_{\psi_{\alpha}},2\mathbb{E}^{1/2}\|X\|_{\infty}^2\big)$ for $\alpha\geq 1$.
 Then, for all $t\geq 0$, with probability at least $1-e^{-t},$
\begin{equation*}
 \biggl\|\frac{1}{n}\sum_{j=1}^nX_j\biggr\|_{\infty}\leq C\bigg(\sigma_X\sqrt{\frac{t+\log(2m)}{n}}\bigvee U_X^{(\alpha)}\log\Big(\frac{U_X^{(\alpha)}}{\sigma_X}\Big)\frac{t+\log(2m)}{n}\bigg).
\end{equation*}
\end{lemma}
The following lemma will be helpful. It means that given $S\in\mathcal{S}_m$ and its support $L$, then the projection of $S_1-S$ onto
$L$ dominates its complement in nuclear norm for all $S_1\in\mathcal{S}_m$. This is due to the unit trace of all density matrices.
\begin{lemma}
\label{lowrankcone}
Let $S\in\mathcal{S}_m$ such that $\text{rank}(S)\leq l$. 
 Then any $S_1\in\mathcal{S}_m$, the following inequality holds
 \begin{equation}
\label{lowrankconeineq1}
  \|S_1-S\|_1\leq 2\sqrt{2l}\|S_1-S\|_2.
 \end{equation}
\end{lemma}
\begin{proof}
 Let $L$ denotes the linear space spanned by the first $l$ eigenvectors of $S$
 and $\mathcal{P}_L,\mathcal{P}_{L}^{\perp}$ be corresponding orthogonal projection operators.
 Then if $\text{rank}(S)\leq l$,
 \begin{equation}
\label{lowrankconeineq3}
  \begin{split}
   \|S\|_1=1=&\|S_1\|_1=\|S_1-S+S\|_1\\
   =&\|\mathcal{P}_L(S_1-S)+\mathcal{P}_L^{\perp}(S_1-S)+S\|_1\\
   \geq&\|\mathcal{P}_L^{\perp}(S_1)+S\|_1-\|\mathcal{P}_L(S_1-S)\|_1\\
   =&\|\mathcal{P}_L^{\perp}(S_1)\|_1+\|S\|_1-\|\mathcal{P}_L(S_1-S)\|_1,
  \end{split}
 \end{equation}
 where the last equality is due to the mutual orthogonality between $S$ and $\mathcal{P}_L^{\perp}(S_1)$. As a consequence,
 \begin{equation*}
  \|\mathcal{P}_L^{\perp}(S_1)\|_1\leq \|\mathcal{P}_L(S_1-S)\|_1
 \end{equation*}
 and $\|S_1-S\|_1\leq 2\|\mathcal{P}_L(S_1-S)\|_1\leq 2\sqrt{2l}\|S_1-S\|_2$ since $\mathcal{P}_L(S_1-S)$ has rank at most $2l$, which proves the inequality (\ref{lowrankconeineq1}).
\end{proof}

\section{Minimax lower bounds}\label{minimaxsec}
In this section, we prove the minimax lower bounds in Schatten $p$-norms of estimating $\rho\in\mathcal{S}_{m,r}$ under the Binomial observation model (Assumption~\ref{trassump}). These bounds hold for all the Schatten $p$-norms with $1\leq p\leq +\infty$. In addition, we also obtain the minimax lower bound for Kullback-Leibler divergence. Note that the bounds in Theorem~\ref{minimaxthm} are equivalent to the bounds proved in \cite{koltchinskii15optimal} under the Gaussian noise model with common variance (by setting $\sigma_{\xi}^2\asymp \frac{1}{Km}$ in \cite[Theorem~4]{koltchinskii15optimal}).
\begin{theorem}\label{minimaxthm}
Under the Binomial observation model (Assumption~\ref{trassump}),
 for all $p\in[1,+\infty]$, there exist constants $c,c'>0$ such that the following bounds hold:
 \begin{equation}\label{minimaxthmineq1}
  \underset{\hat{\rho}}{\inf}\ \underset{\mathcal{S}_{m,r}}{\sup}\  \mathbb{P}_\rho\bigg\{
  \|\hat{\rho}-\rho\|_p\geq c\bigg(\frac{mr^{1/p}}{\sqrt{nK}}\bigwedge \bigg(\frac{m}{\sqrt{nK}}\bigg)^{1-\frac{1}{p}}\bigwedge 1\bigg)\bigg\}\geq c'
 \end{equation}
and
\begin{equation}\label{minimaxthmineq2}
 \underset{\hat{\rho}}{\inf}\ \underset{\rho\in\mathcal{S}_{m,r}}{\sup}\  \mathbb{P}_\rho\bigg\{
  K(\rho\|\hat{\rho})\geq c\bigg(\frac{mr}{\sqrt{nK}}\bigwedge 1\bigg)\bigg\}\geq c'
\end{equation}
where $\inf_{\hat{\rho}}$ is taken as the infimum over all etimators $\hat{\rho}$ based on the data $(X_1,K_1^+),\ldots,(X_n,K_n^+)$.
\end{theorem}
\begin{remark}
Note that $n$ and $K$ are the main interesting parameters in {\it quantum state tomography}. The value of $K$ represents the number of quantum systems prepared for each of the Pauli measurements $X_1,\ldots,X_n$. The value of $n$ is closely related to the number of different Pauli measurements needed to be setup in real experiments\footnote{In real experiments, setting up a Pauli measurement is much more time consuming than producing identically prepared quantum systems.}.
Clearly, if $K$ increases, the bounds (\ref{minimaxthmineq1}) and (\ref{minimaxthmineq2}) become smaller. In the case $K\asymp m$, we get
$$
 \underset{\hat{\rho}}{\inf}\ \underset{\mathcal{S}_{m,r}}{\sup}\  \mathbb{P}_\rho\bigg\{
  \|\hat{\rho}-\rho\|_p\geq c\bigg(\sqrt{\frac{m}{n}}r^{1/p}\bigwedge \bigg(\sqrt{\frac{m}{n}}\bigg)^{1-\frac{1}{p}}\bigwedge 1\bigg)\bigg\}\geq c',
$$
where $\Big(\sqrt{\frac{m}{n}}\Big)^{1-\frac{1}{p}}$ converges to $0$ for any $p>1$ as long as $\frac{m}{n}\to 0$.  It is worthwhile to compare the bounds (\ref{minimaxthmineq1}) and (\ref{minimaxthmineq2}) with the bounds ($30$) and ($31$) in \cite{koltchinskii15optimal} under the trace regression model with bounded responses (where each $X_i$ is used as measurement on only one quantum system). By replacing $n$ with $nK$ in ($30$) and ($31$) in \cite{koltchinskii15optimal}, we immediately end up with bounds (\ref{minimaxthmineq1}) and (\ref{minimaxthmineq2}). If we just focus on the necessary number of identically prepared quantum systems in both models, bounds (\ref{minimaxthmineq1}) and (\ref{minimaxthmineq2}) are essentially equivalent to the bounds ($30$) and ($31$) in \cite{koltchinskii15optimal}.
However, bounds (\ref{minimaxthmineq1}) and (\ref{minimaxthmineq2}) indicate that the necessary number of different Pauli measurements can be significantly reduced in the Binomial observation model (Assumption~\ref{trassump}).
For instance, bound (\ref{minimaxthmineq2}) is nontrivial only when $nK\gtrsim m^2r^2$. Therefore, $nK=O(m^2r^2)$ random Pauli measurements are needed for bounds ($30$) and ($31$) in \cite{koltchinskii15optimal} . Under the uniformly sampling scheme, it means that all the $O(m^2)$ different Pauli measurements will be used. However, in our Binomial observation model (Assumption~\ref{trassump}) with $K\asymp m$, bound (\ref{minimaxthmineq2}) indicates that $O(mr^2)$ different Pauli measurements might be enough to produce an estimation with small error in relative entropy distance. Moreover, if we consider Schatten $p$-norm distances for $p>1$, the Binomial observation model requires only $n=O(m)$ Pauli measurements which is significantly smaller than $m^2$.
\end{remark}

\section{Dantzig estimator and optimal convergence rates}
In view of the minimax lower bounds established in Section~\ref{minimaxsec}, it is natural to ask  which estimators can achieve these convergence rates under the Binomial observation model (Assumption~\ref{trassump}). In \cite{flammia2012quantum}, \cite{gross2010quantum} and \cite{liu2011universal}, the matrix LASSO estimator and Dantzig estimator were considered in the setting $n\gtrsim mr\log^6m$, where the convergence rates in Schatten $p$-norms are obtained for $1\leq p\leq 2$. Those convergence rates match the first term in (\ref{minimaxthmineq1}) up to logarithmic terms\footnote{In their work, they focused on a nontrivial upper bound for the Schatten $1$-norm. It corresponds to $\sqrt{nK}\gtrsim mr$ in (\ref{minimaxthmineq1}), in which case the second term $\Big(\frac{m}{\sqrt{nK}}\Big)^{1-\frac{1}{p}}$ is dominated by the first term $\frac{mr^{1/p}}{\sqrt{nK}}$.}.
In our recent paper \cite{koltchinskii15optimal}, a least squares estimator was studied and the convergence rates in Schatten $p$-norms for $1\leq p\leq 2$ were proved as long as $K\lesssim n$ (noise is large enough). For Schatten $p$-norms with $1\leq p\leq +\infty$, a different estimator was proposed in \cite{Xia2016Estimation} based on the eigenvalues thresholding which achieves convergence rates matching the minimax lower bounds (\ref{minimaxthmineq2}) when $K\lesssim m$. Both the bounds proved in \cite{koltchinskii15optimal} and \cite{Xia2016Estimation} are nontrivial as long as $n\gtrsim_{\log(m,n)} m.$
Clearly, both estimators in \cite{koltchinskii15optimal} and \cite{Xia2016Estimation} have advantages in different settings. In this section, we show that the advantages of both estimators can be obtained simultaneously for the Dantzig-type estimator. Moreover, nontrivial Schatten $p$-norms for $1\leq p\leq +\infty$ can be obtained under weaker conditions. Let's begin with the introduction to the Dantzig-type estimator.
\subsection{Dantzig estimator}
Recall that the central problem in QST is to estimate an unknown high-dimensional low rank density 
matrix $\rho$ based on the data $(X_1,K_1^+),\ldots,(X_n,K_n^+)$ satisfying the Binomial observation model (Assumption~\ref{trassump}). Recall $Y_i:=\frac{K_i^+-K_i^-}{K\sqrt{m}}$ for $1\leq i\leq n$ and as a result $(X_i,Y_i)$ satisfy the trace regression model (\ref{trace_regression}) with $\mathbb{E}_{\rho}\big(\xi_i|X_i\big)=0$ and $\textrm{Var}_{\rho}\big(\xi_i|X_i\big)\leq \frac{1}{Km}$. Moreover, it is easy to check that $\mathbb{E}_{\rho}\Big(e^{Km\xi_i^2}|X_i\Big)\leq C$ for some absolute constant $C>0$ (that is, $\xi_i|X_i$ is a subGaussian random variable).
In the following, we consider estimators based on the data $(X_1,Y_1),\ldots,(X_n,Y_n)$.

The standard matrix Dantzig estimator (or Selector) is defined as the solution to the following convex optimization problem:
\begin{equation}\label{standardDS}
\min \|S\|_1\quad {\rm subject\ to}\ \Big\|\frac{1}{n}\sum_{j=1}^n\big(Y_j-\big<S,X_j\big>\big)X_j\Big\|_{\infty}\leq \epsilon,
\end{equation}
for some $\epsilon\geq 0$. When $\epsilon=0$, it corresponds to the noiseless settings (i.e., $K=+\infty$ which never happens in reality) where the exact recovery of $\rho$ is the main interest.
 It was introduced in Cand\`es and Plan~\cite{candes2011tight} for low rank matrix estimation and was applied in {\it quantum state tomography} for estimating low rank density matrices, see Liu~\cite{liu2011universal},
 Gross~\cite{gross2011recovering} and Flammia et al.~\cite{flammia2012quantum}. They also proved sharp (in this paper, ``sharp" means optimality up to logarithmic factors) convergence rates in Schatten $1$-norm and Schatten $2$-norm distances by applying some techniques based on the \textit{restricted isometry property}(RIP) which requires $n\gtrsim_{\log(m)} mr$ Pauli measurements.
RIP is a strong assumption, but there is yet no results related to its convergence in other Schatten $p$-norms. Moreover, even though the condition $n\gtrsim_{\log m} mr$ looks natural for low rank matrix completion or estimation, this condition might not be necessary for density matrix estimation, especially when we focus on Schatten $p$-norms for $p\neq 1$. The reason is that a density matrix itself essentially has low rank due to its unit trace.
Indeed, our results show that $n\gtrsim_{\log m} m$ is sufficient to produce a consistent estimation in Schatten $p$-norms with $p>1$.

When $S\in\mathcal{S}_m$, the objective function in (\ref{standardDS}) is always $1$ and provides no benefit to the optimization problem.
Instead, 
we study the following estimator:
\begin{equation}
\label{est}
 \hat{\rho}_{\epsilon}:=\arg\min\Bigl\{\text{tr}(S\log S): S\in\Lambda(\epsilon)\Bigr\},
\end{equation}
and
\begin{equation}
\Lambda(\eps):=\Big\{S\in\mathcal{S}_m:  \Big\|\frac{1}{n}\sum_{j=1}^n\big(Y_j-\big<S,X_j\big>\big)X_j\Big\|_{\infty}\leq \epsilon\Big\}
\end{equation}
where we replaced the nuclear norm in (\ref{standardDS}) with negative von Neumann entropy. von Neumann entropy of a density matrix $S$ is defined as
\begin{equation*}
V(S):=-\tr(S\log S),\  \forall S\in\mathcal{S}_m,
\end{equation*}
which is a concave function on $\mathcal{S}_m$ and then (\ref{est}) is actually a convex optimization problem.
von Neumann entropy plays an important role in {\it quantum information} theory and it was used in \cite{koltchinskii2011neumann} and \cite{koltchinskii15optimal} as a penalization to the least squares estimator. 
In this paper, we prove the sharp convergence rates of $\hat{\rho}_{\epsilon}$ in all the Schatten $p$-norms with $p\in[1,+\infty]$. It is easy to show that these rates also hold for the standard matrix Dantzig
estimator (\ref{standardDS}). As a benefit of von Neumann entropy in (\ref{est}), we obtain sharp convergence rate of $\hat{\rho}_{\epsilon}$ in Kullback-Leibler divergence.

\subsection{Oracle inequality and Schatten $p$-norm convergence rates}

Theorem~\ref{oraclethm} displays the performance of $\hat{\rho}_{\epsilon}$ by a \textit{low rank oracle inequality}. The \textit{low rank oracle inequality} 
has been well studied for (matrix) LASSO estimator(see \cite{koltchinskii2011oracle}~\cite{koltchinskii2013sharp},
and \cite{koltchinskii15optimal}).
When studying Dantzig estimator in compressed sensing problems, the {\it sparsity oracle inequality} is considered over all oracles in the feasible set (that is $\Lambda(\epsilon)$ in this paper), for example \cite{koltchinskii2009dantzig}.
It is generally impossible to compare the performance of the estimator with sparse oracles (or low rank oracles in matrix compressed sensing) when they are not in the feasible set.
Surprisingly, we can obtain the following {\it low rank oracle inequality} for $\hat{\rho}_{\epsilon}$ which actually hold for all the oracles in $\mathcal{S}_m$, even when the oracle is infeasible for the optimization problem (\ref{standardDS}) and (\ref{est}).

\begin{theorem}
 \label{oraclethm}
Under the Binomial observation model (Assumption~\ref{trassump}) and assume $\rho\in \mathcal{S}_{m,r}$,
let $\hat{\rho}_{\epsilon}$ be as defined in (\ref{est}). For $\epsilon\geq \frac{C_1}{m}\sqrt{\frac{t+\log(2m)}{nK}}$ with any $t\geq 1$ and some large enough constant $C_1>0$, there exists a constant $C>0$ such that with probability at least $1-e^{-t}$,
\begin{equation}
\label{thmineq1}
\begin{split}
 \|\hat{\rho}_{\epsilon}&-\rho\|_{L_2(\Pi)}^2\leq\underset{S\in\mathcal{S}_m}{\inf}\Big\{2\|S-\rho\|_{L_2(\Pi)}^2\\
 &+ C\Big(m^2\epsilon^2\text{rank}(S)+\frac{\rank(S)\big(t+\log(2m)\big)}{nK}+\frac{\rank(S)\big(t+\log^3m\log^3n\big)^2}{n^2}\Big)
\Big\}.
\end{split} 
\end{equation}
Moreover, if $\epsilon=\frac{C_1}{m}\sqrt{\frac{\log(2m)}{nK}}$, then with probability at least $1-\frac{1}{2m}$,
\begin{equation}
\label{thmineq2}
\begin{split}
 \|\hat{\rho}_{\epsilon}-&\rho\|_{L_2(\Pi)}^2\leq C\Big(\frac{\rank(\rho)\log(2m)}{nK}+\frac{\rank(\rho)\log^6m\log^6n}{n^2}\Big)
\end{split}
\end{equation}
and
\begin{equation}
\label{thmineq3}
\begin{split}
  K(\rho\|\hat{\rho}_{\epsilon})\leq C\frac{\rank(\rho)m\log^{1/2}(2m)\log(Kmn)}{\sqrt{nK}}
  +C\frac{\rank(\rho)m\log^3m\log^3n\log(Kmn)}{n}
  \end{split}
\end{equation}
\end{theorem}
\begin{remark}\label{thmremark}
The objective function in optimization problem (\ref{est}) is not involved in the proof of (\ref{thmineq1}). Therefore, bound (\ref{thmineq1}) also holds for the standard Dantzig estimator (\ref{standardDS}).
Moreover, instead of (\ref{thmineq1}), we actually prove a stronger bound in Section~\ref{proofthmsec}:
\begin{equation*}
\begin{split}
 \|\hat{\rho}_{\epsilon}-&\rho\|_{L_2(\Pi)}^2\leq 2\|S-\rho\|_{L_2(\Pi)}^2+C\frac{\rank(S)\big(t+\log(2m)\big)}{nK}\\
 +&C\Big(m^2\epsilon^2\text{rank}(S)++\frac{\rank(S)\big(t+\log^3m\log^3n\big)^2}{n^2}\big(\|\hat{\rho}_{\epsilon}-S\|_1^2+\|S-\rho\|_1^2\big)\Big),
 \end{split}
\end{equation*}
for any $S\in\mathcal{S}_m$. Consider $S=\rho$, $t=\log(2m)$ and $\eps=\frac{C_1}{m}\sqrt{\frac{\log(2m)}{nK}}$,
it indicates that if $n\geq C'mr\log^3m\log^3n$ for a large enough constant $C'>0$
 such that (due to Lemma~\ref{lowrankcone}) 
 $$
\frac{\rank(\rho)\log^6m\log^6n}{n^2}\|\hat{\rho}_{\epsilon}-\rho\|_1^2\leq \frac{8m^2r^2\log^6m\log^6n}{n^2}\|\hat{\rho}_{\epsilon}-\rho\|_{L_2(\Pi)}^2\leq \frac{1}{2}\|\hat{\rho}_{\epsilon}-\rho\|_{L_2(\Pi)}^2,
 $$
 we get $\|\hat{\rho}_{\epsilon}-\rho\|_{L_2(\Pi)}^2\leq 4C\frac{\rank(\rho)\log(2m)}{nK}$, which reduces to the canonical result by applying the {\it restricted isometry property} (see \cite{liu2011universal},\cite{candes2011tight}).
This bound depends linearly on $\frac{1}{K}$ (which can be arbitrarily small, even $K=+\infty$), see also Remark~\ref{Spremark} after Theorem~\ref{Spthm} and the discussion in Section~\ref{discusssec}.
\end{remark}
An immediate result of Theorem~\ref{oraclethm} is as follows. The following bound holds with probability at least $1-\frac{1}{2m}$,
\begin{equation}
\label{oracleineq1}
\|\hat{\rho}_{\eps}-\rho\|_{L_2(\Pi)}^2\leq C\Big(\frac{r\log(2m)}{nK}+\frac{r\log^6m\log^6n}{n^2}\Big)
\end{equation}
with the choice of $\epsilon=\frac{C}{m}\sqrt{\frac{\log(2m)}{nK}}$. Assume $K\lesssim n$, 
(\ref{oracleineq1}) can be simplified into
\begin{equation}\label{sc1eq0}
 \|\hat{\rho}_{\epsilon}-\rho\|_{2}\lesssim \frac{\sqrt{r}m\log^{3}m\log^3n}{\sqrt{nK}}
\end{equation}
and (due to Lemma~\ref{lowrankcone})
\begin{equation}\label{sc1eq1}
  \|\hat{\rho}_{\epsilon}-\rho\|_{1}\lesssim \frac{rm\log^3m\log^3n}{\sqrt{nK}}
\end{equation}
and
\begin{equation*}
 K(\rho\|\hat{\rho}_{\epsilon})\lesssim \frac{rm\log^3m\log^3n\log(Kmn)}{\sqrt{nK}}.
\end{equation*}
The following corollary is an immediate result by applying a similar approach as in \cite[Theorem~21]{koltchinskii15optimal}. Note that similar convergence rates (with different logarithmic factors) for Schatten $p$-norms with $1\leq p\leq 2$ have been proved for the least squares estimator under the condition $K\lesssim \frac{n}{\log^5m\log^6n}$ in \cite{koltchinskii15optimal}.
\begin{corollary}\label{cor}
Let $\hat{\rho}_{\eps}$ be defined as (\ref{est}) with $\eps:=\frac{C}{m}\sqrt{\frac{\log(2m)}{nK}}$ and assume that $K\lesssim n$. Then, for all $1\leq p\leq 2$, with probability at least $1-\frac{1}{2m}$,
\begin{equation}\label{corineq1}
\|\hat{\rho}_{\eps}-\rho\|_p\leq C\bigg(\frac{mr^{1/p}}{\sqrt{nK}}\big(\log^3m\log^3n\big)^{(2-p)/p}\bigwedge \Big(\frac{m}{\sqrt{nK}}\Big)^{1-\frac{1}{p}}\big(\log^3m\log^3n\big)^{\frac{1}{2}-\frac{1}{2p}}\bigg)\bigwedge 2.
\end{equation}
\end{corollary}

Moreover, we get upper bounds for $\|\hat{\rho}_{\eps}-\rho\|_p$ with $1\leq p\leq +\infty$ if $K\lesssim m$. The bounds in Theorem~\ref{Spthm} are similar to the bounds proved  in \cite{Xia2016Estimation} for a distinct estimator based on eigenvalues thresholding.
Both the bounds (\ref{corineq1}) and (\ref{Spthmineq1}) match, except the logarithmic factors, the minimax lower bounds (\ref{minimaxthmineq1}) in corresponding settings of $K$.
\begin{theorem}
 \label{Spthm}
Let $\hat{\rho}_{\eps}$ be defined as (\ref{est}) with the choice of $\epsilon=\frac{C_1}{m}\sqrt{\frac{\log(2m)}{nK}}$
for some large enough constant $C_1>0$ and assume $K\lesssim m$. Then, there exists a constant $C>0$ such that with probability at least $1-\frac{1}{2m}$,
\begin{equation}
\label{Spthmineq1}
  \|\hat{\rho}_{\epsilon}-\rho\|_p\leq C\biggl(\frac{mr^{1/p}}{\sqrt{nK}}\log^3m\log^3n\bigwedge \Bigl(\frac{m}{\sqrt{nK}}\Bigr)^{1-\frac{1}{p}}\bigl(\log^3m\log^3n\bigr)^{1-\frac{1}{p}}\biggr)\bigwedge 2,
\end{equation}
for all $1\leq p\leq +\infty$. In the case $K\gtrsim m$, the upper bounds still hold by replacing $K$ with $m$.
\end{theorem}
\begin{remark}\label{Spremark}
Basically, Corollary~\ref{cor} and Theorem~\ref{Spthm} indicate that the performances of both estimators in \cite{koltchinskii15optimal} and \cite{Xia2016Estimation} can be achieved simultaneously by the Dantzig-type estimator $\hat{\rho}_{\eps}$. 
Moreover, it worths to point out that a bound stronger than (\ref{Spthmineq1}) is actually proved in Section~\ref{proofthmsec}
\begin{align}\label{Spremarkineq1}
  \|\hat{\rho}_{\epsilon}-\rho\|_p\leq C\biggl(\Big(&\frac{1}{\sqrt{mK}}\vee \frac{\|\hat{\rho}_{\eps}-\rho\|_1}{m}\Big)\frac{m^{3/2}r^{1/p}}{\sqrt{n}}\log^3m\log^3n\nonumber\\
  &\bigwedge \Big(\big(\frac{1}{\sqrt{mK}}\vee \frac{\|\hat{\rho}_{\eps}-\rho\|_1}{m}\big)\frac{m^{3/2}}{\sqrt{n}}\Big)^{1-\frac{1}{p}}\bigl(\log^3m\log^3n\bigr)^{1-\frac{1}{p}}\biggr)\bigwedge 2,
\end{align}
which holds with probability at least $1-\frac{1}{2m}$. Recall that when $K\lesssim n$ from (\ref{sc1eq1}),
$$
 \frac{\|\hat{\rho}_{\epsilon}-\rho\|_{1}}{m}\lesssim \frac{r\log^3m\log^3n}{\sqrt{nK}}\bigwedge \frac{2}{m}.
$$
Then, we get from (\ref{Spremarkineq1}) that
\begin{equation}\label{Spremarkineq2}
  \|\hat{\rho}_{\epsilon}-\rho\|_p\lesssim_{\log(m,n)}\Big(\frac{mr^{1/p}}{\sqrt{nK}}\vee \frac{m^{3/2}r^{1+\frac{1}{p}}}{n\sqrt{K}}\Big)\bigwedge \Big(\frac{m}{\sqrt{nK}}\vee \frac{m^{3/2}r}{n\sqrt{K}}\Big)^{1-\frac{1}{p}}\bigwedge 1
\end{equation}
for all $1\leq p\leq +\infty$. This bound holds for any $n, K$ as long as $K\lesssim n$. It is unclear at this moment whether the least squares estimator in \cite{koltchinskii15optimal} and the projection estimator in \cite{Xia2016Estimation} can obtain this bound in the same settings. Some interesting bounds can be obtained when we consider specific choices of $n$ and $K$. For instance, consider $K\asymp n$ such that $n\asymp\sqrt{nK}\gtrsim mr\log^3m\log^3n$. Then, the following bound holds with probability at least $1-\frac{1}{2m}$,
$$
 \|\hat{\rho}_{\epsilon}-\rho\|_p\lesssim_{\log(m,n)} \Big(\frac{m}{n}r^{1/p}\bigvee \frac{m^{3/2}}{n^{3/2}}r^{1+\frac{1}{p}}\Big)\bigwedge \Big(\frac{m^{3/2}}{n^{3/2}}r\Big)^{1-\frac{1}{p}},\quad 1\leq p\leq +\infty.
$$
More specifically, consider $p=+\infty$, we obtain $\|\hat{\rho}_{\eps}-\rho\|_{\infty}\lesssim_{\log(m,n)}\frac{m^{3/2}}{n^{3/2}}r\lesssim \frac{m\sqrt{r}}{n}\sqrt{\frac{mr}{n}}$. It is interesting to compare this bound with the bounds established for the projection estimator \cite{Xia2016Estimation} and the least squares estimator \cite{koltchinskii15optimal}. Since $K\gtrsim m$, it is proved that the spectral norm convergence rate of the projection estimator (similar bounds as in Theorem~\ref{Spthm}) is of the order $\sqrt{\frac{m}{n}}$ (with logarithmic factors). Clearly, $\frac{m^{3/2}}{n^{3/2}}r\leq \sqrt{\frac{m}{n}}$.
In addition, if we consider the least squares estimator and control its spectral norm convergence rate by its Frobenius norm convergence rate (similar bounds as in Corollary~\ref{cor}), we obtain a simple bound (recall that $K\asymp n$) as $\frac{m\sqrt{r}}{n}$ (up to logarithmic factors) which clearly dominates $\frac{m\sqrt{r}}{n}\sqrt{\frac{mr}{n}}$, especially when $n\gtrsim mr^\alpha$ for $\alpha>1$. 
Basically, we conclude that the Dantzig estimator $\hat{\rho}_{\eps}$ can achieve better convergence rates than the least squares estimator and the projection estimator in the case $n\gtrsim_{\log m} mr$. More discussions are provided in Section~\ref{discusssec}.
\end{remark}

\section{Discussion}\label{discusssec}
The main purpose of this paper is to study the convergence rates of the Dantzig estimator in Schatten $p$-norms for all $1\leq p\leq +\infty$ and compare it with the least squares estimator in \cite{koltchinskii15optimal} and the projection estimator in \cite{Xia2016Estimation}. In this section, we provide a summary of convergence rates in Schatten norm distances obtained for estimators which have been proved in the literature.
The summary is based on the different cases of $n$ and $K$. In the following, consider $n\asymp_{\log(m,n)} mr^{\alpha}$ (assume that $m$ and $r$ dominate the logarithmic factors).
\begin{enumerate}
\item If $0\leq \alpha< 1$. No estimators have been shown to achieve nontrivial convergence rate in Schatten $1$-norm distance. For Schatten $p$-norms with $1<p$, the convergence rates are proved in Corollary~\ref{cor} and Theorem~\ref{Spthm}. These rates can be obtained by the least squares estimator \cite{koltchinskii15optimal} and the projection estimator \cite{Xia2016Estimation} separately, and simultaneously by the Dantzig-type estimator $\hat{\rho}_{\eps}$. Basically, the following bounds  hold with high probability,
$$
\|\hat{\rho}_{\eps}-\rho\|_p\lesssim_{\log(m,n)} \Big(\frac{m}{\sqrt{nK}}\Big)^{1-\frac{1}{p}}\bigwedge 1, \quad 1\leq p\leq 2, K\lesssim n.
$$
and
$$
\|\hat{\rho}_{\eps}-\rho\|_p\lesssim_{\log(m,n)} \Big(\frac{m}{\sqrt{nK}}\Big)^{1-\frac{1}{p}}\bigwedge 1, \quad 1\leq p\leq +\infty, K\lesssim m.
$$
\item If $1\leq \alpha<2$. The RIP technique can be used or the Dantzig estimator can be considered (as in Remark~\ref{thmremark}) and the following bound holds with high probability,
$$
\|\hat{\rho}_{\eps}-\rho\|_p\lesssim_{\log(m,n)} \Big(\frac{mr^{1/p}}{\sqrt{nK}}\bigwedge \Big(\frac{m}{\sqrt{nK}}\Big)^{1-\frac{1}{p}}\bigwedge 1\Big), \quad 1\leq p\leq 2, 1\leq K\leq +\infty.
$$
Assume that $K$ is large enough such that $\sqrt{nK}\gtrsim mr$, then with high probability
\begin{equation}
\label{Spthmineq3}
 \|\hat{\rho}_{\epsilon}-\rho\|_p\lesssim_{\log(m,n)} \Big(\frac{mr^{1/p}}{\sqrt{nK}}\vee \frac{m^{3/2}r^{1+\frac{1}{p}}}{n\sqrt{K}}\Big)\bigwedge \Big(\frac{m^{3/2}r}{n\sqrt{K}}\Big)^{1-\frac{1}{p}}\bigwedge 1,\quad 1\leq p\leq +\infty.
  \end{equation}
Moreover, if $K\asymp n\asymp_{\log(m,n)} mr^{\alpha}$ for $1\leq \alpha<2$, then 
$$
\|\hat{\rho}_{\eps}-\rho\|_{\infty}\lesssim_{\log(m,n)} \frac{m}{\sqrt{nK}}r^{1-\frac{\alpha}{2}}
$$
where the right hand side can be equivalently written as $\sqrt{\frac{m}{n}}\frac{1}{r^{\alpha-1}}$. These bounds are nontrivial even when we compare them with the bounds obtained for the least squares estimator \cite{koltchinskii15optimal} and the projection estimator \cite{Xia2016Estimation}. Note that there is no upper bound constraint on the value of $K$.
  \item If $\alpha\geq 2$. In this case, the convergence rates in Schatten $p$-norms are simple for all $1\leq p\leq +\infty$. Together with the minimax lower bounds in Theorem~\ref{minimaxthm}, 
  $$
  \underset{\hat{\rho}}{\inf}\ \underset{\rho\in\mathcal{S}_{m,r}}{\sup} \mathbb{E}_{\rho}\|\hat{\rho}-\rho\|_p\asymp \Big(\frac{mr^{1/p}}{\sqrt{nK}}\bigwedge 1\Big)
  $$
  for all $1\leq p\leq +\infty$ and $1\leq K\leq +\infty$.
  It is interesting to notice that  the condition $n\gtrsim_{\log(m,n)} mr^2$ is also needed in proving the optimal
  convergence rates in Schatten $p$-norms of Dantzig estimator for estimating general low rank matrices with Gaussian or Rademacher measurements,
  see Xia~\cite{xia2014optimal}. It is still an open problem that whether this condition is necessary.
\end{enumerate}

\section*{Acknowledgment}
I would like to thank an anonymous reviewer for several helpful suggestions on improving the quality of the paper.
\section{Proofs}
\subsection{Proof of the minimax lower bounds}
\begin{proof}[Proof of Theorem~\ref{minimaxthm}]
The proof is based on several steps. Basically, some techniques in \cite{koltchinskii15optimal} are combined and we construct a subset of $\mathcal{S}_p'\subset \mathcal{S}_{m,r}$ whose elements are well separated in Schatten $p$-norm such that for each $\rho\in \mathcal{S}_p$, $\big|\langle\rho, E_k \rangle\big|\leq \frac{0.7}{\sqrt{m}}$ with $k=2,3,\ldots,m^2$. Then, minimax lower bounds are established by calculating the Kullback-Leibler divergence between Binomial distributions.

Denote by $\mathcal{G}_{k,l}$ the Grassmann manifold which is the set of all $k$-dimensional subspaces $L$ of the $l$-dimensional space $\mathbb{R}^l$. Given such a subspace $L\subset\mathbb{R}^l$ with $\dim(L)=k$, let $P_L$ be the orthogonal projector onto $L$ and let $\mathfrak{P}_{k,l}:=\{P_L: L\in\mathcal{G}_{k,l}\}$. The set of all $k$-dimensional projectors $\mathfrak{P}_{k,l}$ will be equipped with Schatten $p$-norm distances for all $p\in[1,+\infty]$ (which can be also viewed as distances on the Grassmannian itself): $d_p(Q_1,Q_2):=\|Q_1-Q_2\|_p, Q_1,Q_2\in\mathfrak{P}_{k,l}$. Recall that the $\eps$-packing number of a metric space $(T,d)$ is defined as
$$
D(T,d,\eps):=\max\Big\{n:\  \textrm{there are } t_1,\ldots,t_n\in T,\ \textrm{such that} \ \min_{i\neq j} d(t_i,t_j)>\eps\Big\}.
$$
The packing number of $\mathfrak{P}_{k,l}$ with respect to Schatten distances $d_p$ will be needed and it is given in the following lemma (see Pajor~\cite{pajor1998metric}).
\begin{lemma}\label{pajorlem}
For all integer $1\leq k\leq l$ such that $k\leq l-k$, and all $1\leq p\leq \infty$, the following bounds hold
\begin{equation}\label{pajorlemeq1}
\Big(\frac{c}{\eps}\Big)^d\leq D\big(\mathfrak{P}_{k,l},d_p,\eps k^{1/p}\big)\leq \Big(\frac{C}{\eps}\Big)^d,\quad \eps>0
\end{equation}
with $d=k(l-k)$ and universal positive constants $c,C$.
\end{lemma}
We will prove the bound (\ref{minimaxthmineq1}) for $2\leq r\leq m/2$ since the proof in the case $r=1$ is simpler. Moreover, the case $r>m/2$ can be easily reduced to the case $r\leq m/2$ by adjusting the constant $c$ in (\ref{minimaxthmineq1}). According to Lemma~\ref{pajorlem}, there is a subset $\mathcal{D}_p\subset \mathfrak{P}_{r-1,m-1}$ such that $\textrm{Card}\big(\mathcal{D}_p\big)\geq 2^{(r-1)(m-r)}$ and, for some positive constant $c'$, $\|Q_1-Q_2\|_p\geq c'(r-1)^{1/p}$ for any $Q_1,Q_2\in \mathcal{D}_p$ and $Q_1\neq Q_2$. Note that for any $Q\in\mathfrak{P}_{r-1,m-1}$, it can be viewed as an $(m-1)\times (m-1)$ positive definite matrix with $\tr(Q)=r-1$. Then, construct the following $m\times m$ density matrix
\begin{equation}\label{SQmat}
S_{Q}=\left(
\begin{array}{cc}
1-\kappa& \bf 0'\\
{\bf 0}&\kappa\frac{Q}{r-1}
\end{array}
\right).
\end{equation}
It is easy to check that whenever $\kappa\leq 1$, $S_Q$ is indeed a density matrix with rank at most $r$. Now we take $\kappa:=c_1\frac{m(r-1)}{\sqrt{nK}}$ with a small enough absolute constant $c_1>0$ and assume that $\kappa\leq \frac{1}{2}$. 

Define a subset of density matrices $\mathcal{S}_p:=\big\{S_Q: Q\in\mathcal{D}_p\big\}$ and an immediate result is $\text{Card}(\mathcal{S}_p)=\text{Card}(\mathcal{D}_p)\geq 2^{(r-1)(m-r)}$ and $\mathcal{S}_p\subset \mathcal{S}_{m,r}$. Moreover, for $S_{Q_1}, S_{Q_2}$ with $Q_1,Q_2\in\mathcal{D}_p$ and $Q_1\neq Q_2$, we have
\begin{align*}
\|S_{Q_1}-&S_{Q_2}\|_p=\frac{\kappa}{r-1}\|Q_1-Q_2\|_p\geq c'\kappa(r-1)^{1/p-1}\\
&\geq c'c_1\frac{m(r-1)^{1/p}}{\sqrt{nK}}\geq c\frac{mr^{1/p}}{\sqrt{nK}}
\end{align*}
for some constant $c>0$. To this end, we obtain a large enough subset $\mathcal{S}_p$ such that each element is well separated in Schatten $p$-norm which holds for any $1\leq p\leq +\infty$. Recall that $\mathcal{E}:=\{E_1,E_2,\ldots,E_{m^2}\}$ is the set of Pauli matrices with $E_1=\frac{I_m}{m}$ with $I_m$ being the $m\times m$ identity matrix. Now, we construct a subset $\mathcal{S}_p'\subset\mathcal{S}_{m,r}$ such that $\text{Card}(\mathcal{S}_p')=\text{Card}(\mathcal{S}_p)$ and for each $S\in\mathcal{S}_p'$, $\big|\langle S,E_k \rangle\big|\leq \frac{0.7}{\sqrt{m}}$ for all $2\leq k\leq m^2$. The following lemma will be needed and its proof is provided in \cite[Lemma~9]{koltchinskii15optimal} by choosing $\gamma=0.2$ there and observing $\tr(E_k)=0$ for $2\leq k\leq m^2$.
\begin{lemma}\label{Evlem}
There exists a universal constant $C_1>0$ such that when $m\geq C_1$, there exists a vector $v\in\mathbb{C}^m$ with $\|v\|=1$ and
$$
\max_{2\leq k\leq m^2} \big|\langle E_kv,v\rangle\big|\leq \frac{0.1}{\sqrt{m}}.
$$
\end{lemma}
We set $e_1:=v$ in Lemma~\ref{Evlem} and construct an orthonormal basis $e_1,e_2,\ldots,e_m$. Let $\vec{e}:=[e_1,e_2,\ldots,e_m]\in \mathbb{C}^{m\times m}$ with $e_i$ being the $i$-th column of $\vec{e}$ for $1\leq i\leq m$.
To this end, we define the subset of density matrices as follows
\begin{equation}\label{Sp'}
\mathcal{S}_p':=\big\{S_Q'=\vec{e}S_Q\vec{e}': S_Q\in\mathcal{S}_p\big\}.
\end{equation}
In other words, $S_Q'$ is obtained by assuming $S_Q$ defined in (\ref{SQmat}) represent linear transformation in basis $\{e_1,\ldots,e_m\}$. Since $\vec{e}$ is an orthonormal matrix,
$$
\|S_{Q_1}'-S_{Q_2}'\|_p=\|S_{Q_1}-S_{Q_2}\|_p\geq c\frac{mr^{1/p}}{\sqrt{nK}},\quad \forall\ Q_1\neq Q_2\in \mathcal{D}_p.
$$
Moreover, for each $E_k, 2\leq k\leq m^2$,
\begin{align*}
\big|\langle S_Q',&E_k \rangle\big|=\Big|(1-\kappa)\langle E_kv,v\rangle+\frac{\kappa}{r-1}\langle Q,E_k\rangle\Big|\\
\leq&(1-\kappa)\big|\langle E_kv,v\rangle\big|+\frac{\kappa}{r-1}\|E_k\|_{\infty}\|Q\|_1\leq (1-\kappa)\frac{0.1}{\sqrt{m}}+\frac{\kappa}{\sqrt{m}}<\frac{0.7}{\sqrt{m}},
\end{align*}
where we used the fact $\kappa\leq \frac{1}{2}$.
Recall that $\mathbb{P}_{\rho}$ denotes the probability distribution of $(X_1, K_1^+),\ldots,(X_n,K_n^+)$ with $X_i$ being uniformly sampled from $\mathcal{E}$ for each $1\leq i\leq n$. We are ready to prove the upper bound of the Kullback-Leibler divergence $D_{KL}(\mathbb{P}_{S_{Q_1}'}\|\mathbb{P}_{S_{Q_2}'})$ for $S_{Q_1}'\neq S_{Q_2}'\in\mathcal{S}_{p}'$. Let $\Pi$ denote the distribution of $X$ which is a uniform distribution over $\mathcal{E}$. Then,
\begin{align}
D_{KL}(\mathbb{P}_{S_{Q_1}'}\|&\mathbb{P}_{S_{Q_2}'})=n\mathbb{E}_{\Pi} D_{KL}\bigg(\textrm{Bin}\Big(K,\frac{1+\sqrt{m}\langle S_{Q_1'}, X\rangle}{2}\Big)\Big\|\textrm{Bin}\Big(K,\frac{1+\sqrt{m}\langle S_{Q_2'}, X\rangle}{2}\Big)\bigg)\nonumber\\
=&\frac{n}{m^2}D_{KL}\bigg(\textrm{Bin}\Big(K,\frac{1+\sqrt{m}\langle S_{Q_1'}, E_1\rangle}{2}\Big)\Big\|\textrm{Bin}\Big(K,\frac{1+\sqrt{m}\langle S_{Q_2'}, E_1\rangle}{2}\Big)\bigg)\label{KLdivineq1}\\
+&\frac{n}{m^2}\sum_{2\leq k\leq m^2}D_{KL}\bigg(\textrm{Bin}\Big(K,\frac{1+\sqrt{m}\langle S_{Q_1'}, E_k\rangle}{2}\Big)\Big\|\textrm{Bin}\Big(K,\frac{1+\sqrt{m}\langle S_{Q_2'}, E_k\rangle}{2}\Big)\bigg)\label{KLdivineq2}
\end{align}
Recall that $\sqrt{m}\langle S, E_1\rangle=1$ for any $S\in\mathcal{S}_m$. As a result, the term in  (\ref{KLdivineq1}) just equals $0$. To deal with (\ref{KLdivineq2}), we need a simple fact of Kullback-Leibler divergence between two Binomial distributions.
\begin{align*}
D_{KL}\Big(\textrm{Bin}(K,p)\|&\textrm{Bin}(K,q)\Big)=K\Big(p\log\frac{p}{q}+(1-p)\log\frac{1-p}{1-q}\Big)\leq 8K(p-q)^2
\end{align*}
where the last inequality holds whenever $p,q\in[3/20,17/20]$. As a result, we obtain
\begin{align*}
D_{KL}(\mathbb{P}_{S_{Q_1}'}\|&\mathbb{P}_{S_{Q_2}'})=\frac{nK}{m^2}\sum_{2\leq k\leq m^2}m\langle S_{Q_1}'-S_{Q_2}', E_k\rangle^2\\
\leq& \frac{nK}{m}\|S_{Q_1}'-S_{Q_2}'\|_2^2\leq \frac{2nK\kappa^2}{m(r-1)}\\
&=2c_1m(r-1)\leq \frac{1}{10}\log\textrm{Card}(\mathcal{S}_p')=\frac{(r-1)(m-r)\log 2}{10},
\end{align*}
where the last inequality holds as long as $c>0$ is small enough. Then, by \cite[Theorem~2.5]{intro}, there exist universal constants $c,c'>0$ such that
\begin{equation}\label{minimaxineq3}
\underset{\hat{\rho}}{\inf}\ \underset{\rho\in\mathcal{S}_p'}{\sup}\ \mathbb{P}_{\rho}\Big(\|\hat{\rho}-\rho\|_p\geq c\frac{mr^{1/p}}{\sqrt{nK}}\Big)\geq c'
\end{equation}
where the bound holds for all $1\leq p\leq +\infty$. Remember that $\mathcal{S}_p'\subset\mathcal{S}_{m,r}$, we get the first term on the right hand side of bound (\ref{minimaxthmineq1}). To this end, we assumed $\kappa\leq \frac{1}{2}$.

Now, consider $\kappa>\frac{1}{2}$. In this case, $c_1\frac{m(r-1)}{\sqrt{nK}}>\frac{1}{2}$. Choose the largest integer $2\leq r'<r-1$ such that $c_1\frac{m(r'-1)}{\sqrt{nK}}\leq \frac{1}{2}$. Following the method above, we get
\begin{equation}\label{minimaxineq4}
\underset{\hat{\rho}}{\inf}\ \underset{\rho\in\mathcal{S}_{m,r'}}{\sup}\ \mathbb{P}_{\rho}\Big(\|\hat{\rho}-\rho\|_p\geq c\frac{m(r'-1)^{1/p}}{\sqrt{nK}}\Big)\geq c'.
\end{equation}
The definition of $r'$ implies that $r'\asymp r'-1\asymp \Big(\frac{m}{\sqrt{nK}}\Big)^{-1}$. Therefore,
$$
\frac{m(r'-1)^{1/p}}{\sqrt{nK}}\asymp \Big(\frac{m}{\sqrt{nK}}\Big)^{1-1/p}.
$$
Since $\mathcal{S}_{m,r'}\subset\mathcal{S}_{m,r}$, by combing (\ref{minimaxineq3}) and (\ref{minimaxineq4}), we get the bound (\ref{minimaxthmineq1}). Following the similar approach used in \cite[Theorem~4]{koltchinskii15optimal} (by comparing $K(S_1\|S_2)$ with squared Hellinger distance), we can get the bound (\ref{minimaxthmineq2}).
\end{proof}

\subsection{Proof of Theorem~\ref{oraclethm}}\label{proofthmsec}
The main technical tool for our proof is the following lemma which gives a probabilistic upper bound of the product empirical processes.
For any $\Delta\in[0,1]$, define the set and quantity
\begin{equation*}
 \mathcal{A}(\Delta):=\bigl\{A\in\mathbb{H}_m, \|A\|_1\leq 1, \|A\|_{L_2(\Pi)}\leq \Delta\bigr\}
\end{equation*}
and
\begin{equation*}
\alpha_n(\Delta_1,\Delta_2):=\underset{A_1\in\mathcal{A}(\Delta_1)}{\sup}\underset{A_2\in\mathcal{A}(\Delta_2)}{\sup} \Big|\frac{1}{n}\sum_{i=1}^n\big<A_1,X_i\big>\big<A_2,X_i\big>-\mathbb{E}\big<A_1,X\big>\big<A_2,X\big>\Big|.
\end{equation*}
\begin{lemma}
\label{f2lemma}
Let $X_1,\ldots,X_n$ be $i.i.d.$ random matrices uniformly sampled from the Pauli basis $\mathcal{E}$.
 Given $0<\delta^-<\delta^+$ and $t\geq 1$, let
\begin{equation*}
 \bar{t}:=t+\log(\log_2(\delta^+/\delta^-)+3).
\end{equation*}
Then, with some constant $C$ and probability at least $1-e^{-t}$, the following bound holds for all $\frac{\Delta_1+\Delta_2}{2}\in[\delta^-,\delta^+]$:
\begin{equation*}
 \alpha_n(\Delta_1,\Delta_2)\leq C\Bigl[(\Delta_1+\Delta_2) \frac{\log^{3/2}m\log^{3/2}{n}+\sqrt{\bar{t}}}{\sqrt{nm}}+\frac{\log^3m\log^{3}n+\bar{t}}{nm}\Bigr].
\end{equation*}
\end{lemma}
 Generally, tight upper bounds (generic chaining bounds) of product empirical processes are not easy to derive due to the nontrivial geometric structure of the indexing classes of the empirical process, see Mendelson~\cite{mendelson2014upper} and references therein. Even though we suspect that the bound in Lemma~\ref{f2lemma} might not be sharp,  it is sufficient for us to prove the results we need in this paper.
 Lemma~\ref{f2lemma} will be used to prove the oracle inequality (\ref{thmineq1}) and the spectral norm (i.e., $p=+\infty$) convergence rate of $\hat{\rho}_{\epsilon}$ in (\ref{Spthmineq1}).
The proof of Lemma~\ref{f2lemma} is given in Section~\ref{f2lemmasec}.

\begin{proof}[Proof of Theorem~\ref{oraclethm}]
Denote $\Xi_1=\frac{1}{n}\sum_{i=1}^n\xi_iX_i$. By Lemma~\ref{matBernlem}, we know that with probability at least $1-e^{-t}$,
\begin{equation}\label{Xibound}
 \big\|\Xi_1\big\|_{\infty}\leq C\Big(\frac{1}{m}\sqrt{\frac{t+\log(2m)}{nK}}+\frac{\big(t+\log(2m)\big)\log^{1/2}(2m)}{nm\sqrt{K}}\Big)
\end{equation}
for some constant $C>0$. We used the simple facts $\|\mathbb{E}\xi^2 X^2\|_{\infty}^{1/2}\leq \frac{1}{m\sqrt{K}}$ (see \cite{koltchinskii15optimal}) and $\big\|\|\xi X\|_{\infty}\big\|_{\psi_2} \leq \frac{\|\xi\|_{\psi_2}}{\sqrt{m}}\lesssim \frac{1}{m\sqrt{K}}$.
The second term in (\ref{Xibound}) is clearly dominated by the first term as long as $n\geq \big(t+\log(2m)\big)\log(2m)$, which is assumed to be true hereandafter. In the case $n\leq \big(t+\log(2m)\big)\log(2m)$, bounds (\ref{thmineq1}), (\ref{thmineq2}) and (\ref{thmineq3}) are trivial. 

The choice of $\epsilon$ in Theorem~\ref{oraclethm} satisfies that $\eps\gtrsim \|\Xi_1\|_{\infty}$ which guarantees the existence of the solution $\hat{\rho}_{\epsilon}$ since $\Lambda(\epsilon)$ is nonempty and $\rho\in\Lambda(\epsilon)$.
The fact $\hat{\rho}_{\epsilon}\in\Lambda(\epsilon)$ indicates that, for any $S\in\mathcal{S}_m$,
\begin{equation*}
 \frac{1}{n}\sum_{j=1}^n\big(\big<\hat{\rho}_{\epsilon},X_j\big>-Y_j\big)\big<\hat{\rho}_{\epsilon}-S,X_j\big>\leq \epsilon\|\hat{\rho}_{\epsilon}-S\|_1.
\end{equation*}
By arranging the terms accordingly,
\begin{equation*}
\begin{split}
 \big<\hat{\rho}_{\epsilon}-\rho,&\hat{\rho}_{\epsilon}-S\big>_{L_2(\Pi)}\leq \epsilon\|\hat{\rho}_{\epsilon}-S\|_1+\big<\Xi_1,\hat{\rho}_{\epsilon}-S\big>\\
+&\Big|\frac{1}{n}\sum_{i=1}^n\big<\hat{\rho}_{\epsilon}-\rho,X_i\big>\big<\hat{\rho}_{\epsilon}-S,X_i\big>-\mathbb{E}\big<\hat{\rho}_{\epsilon}-\rho,X\big>\big<\hat{\rho}_{\epsilon}-S,X\big>\Big|.
\end{split}
\end{equation*}
Observe that
\begin{equation*}
 2\big<\hat{\rho}_{\epsilon}-\rho,\hat{\rho}_{\epsilon}-S\big>_{L_2(\Pi)}=\|\hat{\rho}_{\epsilon}-\rho\|_{L_2(\Pi)}^2-\|S-\rho\|_{L_2(\Pi)}^2+\|\hat{\rho}_{\epsilon}-S\|_{L_2(\Pi)}^2.
\end{equation*}
Therefore, we get
\begin{equation}\label{proofthmineq1}
\begin{split}
 \|\hat{\rho}_{\epsilon}-\rho\|_{L_2(\Pi)}^2+&\|\hat{\rho}_{\epsilon}-S\|_{L_2(\Pi)}^2\leq \|S-\rho\|_{L_2(\Pi)}^2+2(\epsilon+\|\Xi_1\|_{\infty})\|\hat{\rho}_{\epsilon}-S\|_1\\
+&2\Big|\frac{1}{n}\sum_{i=1}^n\big<\hat{\rho}_{\epsilon}-\rho,X_i\big>\big<\hat{\rho}_{\epsilon}-S,X_i\big>-\mathbb{E}\big<\hat{\rho}_{\epsilon}-\rho,X\big>\big<\hat{\rho}_{\epsilon}-S,X\big>\Big|.
\end{split}
\end{equation}
By definition of $\alpha_n(\Delta_1,\Delta_2)$,
\begin{equation*}
\begin{split}
 \Big|\frac{1}{n}\sum_{i=1}^n\big<\hat{\rho}_{\epsilon}-\rho,X_i\big>&\big<\hat{\rho}_{\epsilon}-S,X_i\big>-\mathbb{E}\big<\hat{\rho}_{\epsilon}-\rho,X\big>\big<\hat{\rho}_{\epsilon}-S,X\big>\Big|\\
\leq& \|\hat{\rho}_{\epsilon}-\rho\|_1\|\hat{\rho}_{\epsilon}-S\|_1\alpha_n\Big(\frac{\|\hat{\rho}_{\epsilon}-\rho\|_{L_2(\Pi)}}{\|\hat{\rho}_{\epsilon}-\rho\|_1},\frac{\|\hat{\rho}_{\epsilon}-S\|_{L_2(\Pi)}}{\|\hat{\rho}_{\epsilon}-S\|_1}\Big).
\end{split}
\end{equation*}
We apply Lemma~\ref{f2lemma} with $\delta^-=\frac{1}{mn}$ and $\delta^+=\frac{1}{m}$. Clearly, if $\frac{\|\hat{\rho}_{\epsilon}-\rho\|_{L_2(\Pi)}}{\|\hat{\rho}_{\epsilon}-\rho\|_1}+\frac{\|\hat{\rho}_{\epsilon}-S\|_{L_2(\Pi)}}{\|\hat{\rho}_{\epsilon}-S\|_1}\geq \delta^-$,
Lemma~\ref{f2lemma} yields that, with probability at least $1-e^{-t}$,
\begin{eqnarray}
&\nonumber \Big|\frac{1}{n}\sum_{i=1}^n\big<\hat{\rho}_{\epsilon}-\rho,X_i\big>\big<\hat{\rho}_{\epsilon}-S,X_i\big>-\mathbb{E}\big<\hat{\rho}_{\epsilon}-\rho,X\big>\big<\hat{\rho}_{\epsilon}-S,X\big>\Big|\\
&\nonumber\leq \|\hat{\rho}_{\epsilon}-\rho\|_1\|\hat{\rho}_{\epsilon}-S\|_1\Big(\frac{\|\hat{\rho}_{\epsilon}-\rho\|_{L_2(\Pi)}}{\|\hat{\rho}_{\epsilon}-\rho\|_1}+\frac{\|\hat{\rho}_{\epsilon}-S\|_{L_2(\Pi)}}{\|\hat{\rho}_{\epsilon}-S\|_1}\Big)C\frac{\log^{3/2}m\log^{3/2}n+\sqrt{\bar{t}}}{\sqrt{nm}}\\
&\nonumber+\|\hat{\rho}_{\epsilon}-\rho\|_1\|\hat{\rho}_{\epsilon}-S\|_1C\frac{\log^{3}m\log^{3}n+\bar{t}}{nm}\\
&\nonumber=\|\hat{\rho}_{\epsilon}-S\|_1\|\hat{\rho}_{\epsilon}-\rho\|_{L_2(\Pi)}C\frac{\log^{3/2}m\log^{3/2}n+\sqrt{\bar{t}}}{\sqrt{nm}}\\
&\nonumber+\|\hat{\rho}_{\epsilon}-\rho\|_1\|\hat{\rho}_{\epsilon}-S\|_{L_2(\Pi)}C\frac{\log^{3/2}m\log^{3/2}n+\sqrt{\bar{t}}}{\sqrt{nm}}\\
&\nonumber+\|\hat{\rho}_{\epsilon}-\rho\|_1\|\hat{\rho}_{\epsilon}-S\|_1C\frac{\log^{3}\log^{3}n+\bar{t}}{nm},
\end{eqnarray}
where $\bar{t}=t+\log(\log_2n+3)$. Recall from Lemma~\ref{lowrankcone} that $\|\hat{\rho}_{\epsilon}-S\|_1\leq 2\sqrt{2\text{rank}(S)}\|\hat{\rho}_{\epsilon}-S\|_2$,
\begin{equation*}
\begin{split}
 \|\hat{\rho}_{\epsilon}-S\|_1&\|\hat{\rho}_{\epsilon}-\rho\|_{L_2(\Pi)}C\frac{\log^{3/2}m\log^{3/2}n+\sqrt{\bar{t}}}{\sqrt{nm}}\\
\leq& \frac{1}{4}\|\hat{\rho}_{\epsilon}-\rho\|_{L_2(\Pi)}^2+2C^2\|\hat{\rho}_{\epsilon}-S\|_1^2\frac{\log^3m\log^3n+t}{nm}\\
\leq & \frac{1}{4}\|\hat{\rho}_{\epsilon}-\rho\|_{L_2(\Pi)}^2+\frac{1}{4}\|\hat{\rho}_{\epsilon}-S\|_{L_2(\Pi)}^2+C\frac{\rank(S)\big(t+\log^3m\log^3n\big)^2}{n^2}\|\hat{\rho}_{\epsilon}-S\|_1^2,
\end{split}
\end{equation*}
for some constant $c_1>0$, where we applied the inequality $ab\leq \frac{a^2}{4}+b^2$ multiple times.
Moreover, since $\|\hat{\rho}_{\epsilon}-\rho\|_1\leq \|\hat{\rho}_{\epsilon}-S\|_1+\|S-\rho\|_1$,
\begin{equation*}
 \begin{split}
  \|\hat{\rho}_{\epsilon}-\rho\|_1&\|\hat{\rho}_{\epsilon}-S\|_{L_2(\Pi)}C\frac{\log^{3/2}m\log^{3/2}n+\sqrt{\bar{t}}}{\sqrt{nm}}\\
\leq&\frac{1}{8}\|\hat{\rho}_{\epsilon}-S\|_{L_2(\Pi)}^2+4C^2\|\hat{\rho}_{\epsilon}-\rho\|_1^2\frac{\log^3m\log^3n+t}{nm}\\
\leq&\frac{1}{8}\|\hat{\rho}_{\epsilon}-S\|_{L_2(\Pi)}^2+8C^2\|\hat{\rho}_{\epsilon}-S\|_1^2\frac{\log^3m\log^3n+t}{nm}+8C^2\|S-\rho\|_1^2\frac{\log^3m\log^3n+t}{nm}\\
\leq&\frac{1}{4}\|\hat{\rho}_{\epsilon}-S\|_{L_2(\Pi)}^2+\frac{1}{4}\|S-\rho\|_{L_2(\Pi)}^2+C\frac{\rank(S)\big(t+\log^3m\log^3n\big)^2}{n^2}\big(\|\hat{\rho}_{\epsilon}-S\|_1^2+\|S-\rho\|_1^2\big).
 \end{split}
\end{equation*}
Similarly, we can get
\begin{equation*}
\begin{split}
 \|\hat{\rho}_{\epsilon}-\rho\|_1\|\hat{\rho}_{\epsilon}-S\|_1C&\frac{\log^{3}m\log^{3}n+\bar{t}}{nm}\leq \frac{1}{4}\|\hat{\rho}_{\epsilon}-S\|_{L_2(\Pi)}^2\\
 &+C\frac{\rank(S)\big(t+\log^3m\log^3n\big)^2}{n^2}\big(\|\hat{\rho}_{\epsilon}-S\|_1^2+\|S-\rho\|_1^2\big).
 \end{split}
\end{equation*}
Therefore, we conclude that if $\frac{\|\hat{\rho}_{\epsilon}-\rho\|_{L_2(\Pi)}}{\|\hat{\rho}_{\epsilon}-\rho\|_1}+\frac{\|\hat{\rho}_{\epsilon}-S\|_{L_2(\Pi)}}{\|\hat{\rho}_{\epsilon}-S\|_1}\geq \delta^-$, with probability at least $1-e^{-t}$,
\begin{equation}\label{proofthmineq2}
 \begin{split}
   \Big|\frac{1}{n}\sum_{i=1}^n\big<\hat{\rho}_{\epsilon}-&\rho,X_i\big>\big<\hat{\rho}_{\epsilon}-S,X_i\big>-\mathbb{E}\big<\hat{\rho}_{\epsilon}-\rho,X\big>\big<\hat{\rho}_{\epsilon}-S,X\big>\Big|\\
\leq& \frac{3}{4}\|\hat{\rho}_{\epsilon}-S\|_{L_2(\Pi)}^2+\frac{1}{4}\|\hat{\rho}_{\epsilon}-\rho\|_{L_2(\Pi)}^2+\frac{1}{4}\|S-\rho\|_{L_2(\Pi)}^2\\
&+C\frac{\rank(S)\big(t+\log^3m\log^3n\big)^2}{n^2}\big(\|\hat{\rho}_{\epsilon}-S\|_1^2+\|S-\rho\|_1^2\big).
 \end{split}
\end{equation}
If, on the other hand, $\frac{\|\hat{\rho}_{\epsilon}-\rho\|_{L_2(\Pi)}}{\|\hat{\rho}_{\epsilon}-\rho\|_1}+\frac{\|\hat{\rho}_{\epsilon}-S\|_{L_2(\Pi)}}{\|\hat{\rho}_{\epsilon}-S\|_1}\leq \delta^-=\frac{1}{mn}$, then the proof of (\ref{thmineq1}) only simplifies
since
\begin{equation*}
\begin{split}
\|\hat{\rho}_{\epsilon}-\rho\|_{L_2(\Pi)}^2\leq \frac{1}{n^2m^2}\|\hat{\rho}_{\epsilon}-\rho\|_1^2\leq& \frac{\log^3m\log^3n+t}{n^2}\big(\|\hat{\rho}_{\epsilon}-S\|_1^2+\|S-\rho\|_1^2\big)\\
\leq& \frac{\rank(S)\big(t+\log^3m\log^3n\big)^2}{n^2}\big(\|\hat{\rho}_{\epsilon}-S\|_1^2+\|S-\rho\|_1^2\big).
\end{split}
\end{equation*}
Plugging (\ref{proofthmineq2}) into (\ref{proofthmineq1}), we get that with probability at least $1-e^{-t}$,
\begin{equation}\label{proofthmineq3}
\begin{split}
 \frac{3}{4}\|\hat{\rho}_{\epsilon}-\rho\|_{L_2(\Pi)}^2+\frac{1}{4}&\|\hat{\rho}_{\epsilon}-S\|_{L_2(\Pi)}^2\leq \frac{5}{4}\|S-\rho\|_{L_2(\Pi)}^2+2(\epsilon+\|\Xi_1\|_{\infty})\|\hat{\rho}_{\epsilon}-S\|_1\\
 &+c_1\frac{\rank(S)\big(t+\log^3m\log^3n\big)^2}{n^2}\big(\|\hat{\rho}_{\epsilon}-S\|_1^2+\|S-\rho\|_1^2\big).
 \end{split}
\end{equation}
By the bound (\ref{Xibound}) and the choice of $\epsilon$,
\begin{equation}\label{proofthmineq4}
\begin{split}
 2(\epsilon+\|\Xi_1\|_{\infty})&\|\hat{\rho}_{\epsilon}-S\|_1\leq \frac{1}{4}\|\hat{\rho}_{\epsilon}-S\|_{L_2(\Pi)}^2+4m^2\text{rank}(S) (\epsilon+\|\Xi_1\|_{\infty})^2\\
\leq&\frac{1}{4}\|\hat{\rho}_{\epsilon}-S\|_{L_2(\Pi)}^2+Cm^2\epsilon^2\text{rank}(S)+C\frac{\rank(S)(t+\log(2m))}{nK}.
\end{split}
\end{equation}
By putting the bound (\ref{proofthmineq4}) into (\ref{proofthmineq3}) and adjusting some constants, we get (\ref{thmineq1}). Then (\ref{thmineq2}) is an immediate result from (\ref{thmineq1}) by setting $S=\rho$.

We are ready to prove (\ref{thmineq3}). Consider $\rho'=(1-\delta)\rho+\delta\frac{I_m}{m}$ with $\delta=\frac{1}{n\sqrt{mK}}\leq \frac{1}{n\sqrt{m}}$. It is easy to check that $\rho'\in\Lambda(\epsilon)$ as long as the constant $C_1$ in $\epsilon$ is large enough.
By definition of $\hat{\rho}_{\epsilon}$
(the subdifferential of function $\text{tr}(S\log S)$ at $\hat{\rho}_{\epsilon}$ is $\log(\hat{\rho}_{\epsilon})+I_m$, see \cite{koltchinskii2011neumann}), we get
\begin{equation*}
 \bigl<\log\hat{\rho}_{\epsilon},\hat{\rho}_{\epsilon}-\rho'\bigr>\leq 0,
\end{equation*}
since $\bigl<I_m,\hat{\rho}_{\epsilon}-\rho'\bigr>=0$. Meanwhile, suppose $r=\text{rank}(\rho)$ and $\rho=\sum_{i=1}^r\lambda_j P_j$ with eigenvalues $\lambda_j\in[0,1]$(repeated with their multiplicities) and one dimensional orthogonal eigenprojectors $P_j$.
We extend $P_j,j=1,\ldots,r$ to the complete orthogonal resolution of the identity $P_j, j=1,\ldots,m$. Then
\begin{equation*}
 \begin{split}
  \log\rho'=&\log\Bigl(\bigl(1-\delta\bigr)\rho+\delta\frac{I_m}{m}\Bigr)=\sum_{i=1}^r\log\Bigl((1-\delta)\lambda_j+\delta/m\Bigr)+\sum_{j=r+1}^m\log(\delta/m)P_j\\
=&\sum_{j=1}^r\log\Bigl(1+(1-\delta)m\lambda_j/\delta\Bigr)P_j+\log(\delta/m)I_m.
 \end{split}
\end{equation*}
Therefore,
\begin{equation*}
\begin{split}
 K(\hat{\rho}_{\epsilon};\rho')\leq& -\bigl<\log\rho',\hat{\rho}_{\epsilon}-\rho'\bigr>=\Bigl<\sum_{j=1}^r\log\Bigl(1+(1-\delta)m\lambda_j/\delta\Bigr)P_j,\hat{\rho}_{\epsilon}-\rho'\Bigr>\\
\leq& \Bigl\|\sum_{j=1}^r\log\Bigl(1+(1-\delta)m\lambda_j/\delta\Bigr)P_j\Bigr\|_2\|\hat{\rho}_{\epsilon}-\rho'\|_2
\end{split}
\end{equation*}
Note that $\|\hat{\rho}_{\epsilon}-\rho'\|_2\leq \|\hat{\rho}_{\epsilon}-\rho\|_2+\delta\|\rho-I_m/m\|_2\leq \|\hat{\rho}_{\epsilon}-\rho\|_2+2\delta$ and
\begin{equation*}
\begin{split}
\Bigl\|\sum_{j=1}^r\log\Bigl(1+&(1-\delta)m\lambda_j/\delta\Bigr)P_j\Bigr\|_2=\Bigl(\sum_{j=1}^r\log^2\bigl(1+(1-\delta)m\lambda_j/\delta\bigr)\Bigr)^{1/2}\\
\leq&\sqrt{r}\log(m/\delta).
\end{split}
\end{equation*}
Then, together with (\ref{thmineq2}) and $\delta=\frac{1}{n\sqrt{mK}}$,
\begin{equation}\label{proofthmineq5}
\begin{split}
 K(\hat{\rho}_{\epsilon};\rho')\leq& 2\sqrt{r}\Big(\|\hat{\rho}_{\epsilon}-\rho\|_2+2\frac{1}{n\sqrt{mK}}\Big)\log(Kmn)\\
\leq& C\frac{mr\log^{1/2}(2m)\log(Kmn)}{\sqrt{nK}}+C\frac{mr\log^3m\log^3n\log(Kmn)}{n}+
4\frac{\sqrt{r}}{n\sqrt{mK}}\log(Kmn).
\end{split}
\end{equation}
Recall that $K(\rho'\|\hat{\rho}_{\epsilon})\leq K(\hat{\rho}_{\epsilon};\rho')$ and
the following lemma(see \cite{koltchinskii15optimal})
\begin{lemma}
\label{KLlemma}
 Let $\delta\in(0,1)$, $\rho\in\mathcal{S}_{m,r}$ and $\rho'=(1-\delta)\rho+\delta\frac{I_m}{m}$. Then for any $S\in\mathcal{S}_m$,
\begin{equation*}
 K(\rho\| S)\leq \frac{K(\rho'\| S)+h(\delta)}{1-\delta}
\end{equation*}
with $h(\delta):=\delta\log\frac{1}{\delta}+(1-\delta)\log\frac{1}{1-\delta}\leq \delta\log\frac{e}{\delta}$.
\end{lemma}
, we get $K(\rho\|\hat{\rho}_{\epsilon})\leq 2K(\rho'\|\hat{\rho}_{\epsilon})+\frac{4}{n\sqrt{mK}}\log(eKmn)$. Replacing $K(\rho'\|\hat{\rho}_{\epsilon})$ with the right hand side of (\ref{proofthmineq5}) and observing that $\frac{4\sqrt{r}}{n\sqrt{mK}}\log(Kmn)$ is dominated by the other two terms, we obtain (\ref{thmineq3}).
\end{proof}

\subsection{Proof of Theorem~\ref{Spthm}}
We begin with the proof of the spectral norm $\|\hat{\rho}_{\epsilon}-\rho\|_{\infty}$. Note that
\begin{equation*}
 \frac{\|\hat{\rho}_{\epsilon}-\rho\|_{\infty}}{m^2}\leq \Big\|\frac{1}{n}\sum_{i=1}^n\big<\hat{\rho}_{\epsilon}-\rho,X_i\big>X_i\Big\|_{\infty}+\Big\|\frac{1}{n}\sum_{i=1}^n\big<\hat{\rho}_{\epsilon}-\rho,X_i\big>X_i-\mathbb{E}\big<\hat{\rho}_{\epsilon}-\rho,X\big>X\Big\|_{\infty}.
\end{equation*}
The first term is upper bounded by $2\epsilon=\frac{C_1}{m}\sqrt{\frac{\log(2m)}{nK}}$ with probability at least $1-\frac{1}{2m}$, since $\hat{\rho}_{\epsilon}\in\Lambda(\epsilon)$ and,
\begin{equation*}
 \Big\|\frac{1}{n}\sum_{i=1}^n\big<\hat{\rho}_{\epsilon}-\rho,X_i\big>X_i\Big\|_{\infty}\leq \epsilon+\|\Xi_1\|_{\infty}.
\end{equation*}
By definition of spectral norm, the second term is written as follows (recall the definition of $\mathcal{A}(\Delta)$ in Lemma~\ref{f2lemma}):
\begin{equation}\label{proofSpthmineq1}
\begin{split}
 \Big\|\frac{1}{n}\sum_{i=1}^n\big<\hat{\rho}_{\epsilon}-&\rho,X_i\big>X_i-\mathbb{E}\big<\hat{\rho}_{\epsilon}-\rho,X\big>X\Big\|_{\infty}\\
 =&\underset{V\in\mathcal{A}(\frac{1}{m})}{\sup}\Big|\frac{1}{n}\sum_{i=1}^n\big<\hat{\rho}_{\epsilon}-\rho,X_i\big>\big<V,X_i\big>-\mathbb{E}\big<\hat{\rho}_{\epsilon}-\rho,X\big>\big<V,X\big>\Big|\\
 \leq&\|\hat{\rho}_{\epsilon}-\rho\|_1\alpha_n\Big(\frac{\|\hat{\rho}_{\epsilon}-\rho\|_{L_2(\Pi)}}{\|\hat{\rho}_{\epsilon}-\rho\|_1},\frac{1}{m}\Big).
 \end{split}
\end{equation}
To this end, we apply Lemma~\ref{f2lemma} with $\delta^-=\frac{1}{2m}$ and $\delta^+=\frac{1}{m}$. Then,
\begin{equation*}
 \alpha_n\Big(\frac{\|\hat{\rho}_{\epsilon}-\rho\|_{L_2(\Pi)}}{\|\hat{\rho}_{\epsilon}-\rho\|_1},\frac{1}{m}\Big)\leq C\Big(\frac{1}{m}\frac{\log^{3/2}m\log^{3/2}n}{\sqrt{nm}}+\frac{\log^3m\log^3n}{nm}\Big),
\end{equation*}
which holds with probability at least $1-\frac{1}{2m}$. By simply replacing $\|\hat{\rho}_{\epsilon}-\rho\|_1$ with $2$ in (\ref{proofSpthmineq1}), with probability at least $1-\frac{1}{2m}$,
\begin{equation*}
\begin{split}
 \frac{\|\hat{\rho}_{\epsilon}-\rho\|_{\infty}}{m^2}\leq& C\Big(\big(\frac{1}{\sqrt{mK}}\vee \frac{1}{m}\big)\frac{\log^{3/2}m\log^{3/2}n}{\sqrt{nm}}+\frac{1}{m}\frac{\log^3m\log^3n}{n}\Big)\\
&\leq \frac{C}{\sqrt{mK}}\frac{\log^3m\log^3n}{\sqrt{nm}}=\frac{C}{m}\frac{\log^3m\log^3n}{\sqrt{nK}}
\end{split}
\end{equation*}
where we used the condition $K\lesssim m$ (logarithmic terms get higher order due to this simplifying step). Since $\|\hat{\rho}_{\epsilon}-\rho\|_{\infty}$ has a trivial upper bound $2$, we conclude that
\begin{equation*}
 \|\hat{\rho}_{\epsilon}-\rho\|_{\infty}\leq \frac{Cm\log^3m\log^3n}{\sqrt{nK}}\bigwedge 2.
\end{equation*}
By Lemma~\ref{lowrankcone}, with the same probability,
\begin{equation*}
 \|\hat{\rho}_{\epsilon}-\rho\|_{1}\leq 2\|\mathcal{P}_L(\hat{\rho}_{\epsilon}-\rho)\|_1\leq 2r\|\hat{\rho}_{\epsilon}-\rho\|_{\infty}\leq \frac{Cmr\log^3m\log^3n}{\sqrt{nK}}\bigwedge 2,
\end{equation*}
where $L$ denotes the support of $\rho$.
Applying the {\it interpolation inequality} from Lemma~\ref{interlem},
\begin{equation*}
 \|\hat{\rho}_{\epsilon}-\rho\|_p\leq \|\hat{\rho}_{\epsilon}-\rho\|_1^{1/p}\|\hat{\rho}_{\epsilon}-\rho\|_{\infty}^{1-1/p}
\end{equation*}
for all $1\leq p\leq +\infty$, we get bound (\ref{Spthmineq1}).

\subsection{Proof of Lemma~\ref{f2lemma}}\label{f2lemmasec}
\begin{proof}[Proof of Lemma~\ref{f2lemma}]
For any $\Delta\in[0,1]$, define the following quantity
\begin{equation*}
 \beta_n(\Delta):=\underset{A\in\mathcal{A}(\Delta)}{\sup}\Big|\frac{1}{n}\sum_{i=1}^n\big<A,X_i\big>^2-\mathbb{E}\big<A,X\big>^2\Big|.
\end{equation*}
For all $A_1\in\mathcal{A}(\Delta_1)$ and $A_2\in\mathcal{A}(\Delta_2)$, the following fact is clear,
\begin{equation*}
\begin{split}
 \Big|\frac{1}{n}\sum_{i=1}^n\big<A_1,X_i\big>&\big<A_2,X_i\big>-\mathbb{E}\big<A_1,X\big>\big<A_2,X\big>\Big|\\
\leq&\frac{1}{4}\Big|\frac{1}{n}\sum_{i=1}^n\big<A_1+A_2,X_i\big>^2-\mathbb{E}\big<A_1+A_2,X\big>^2\Big|\\
+&\frac{1}{4}\Big|\frac{1}{n}\sum_{i=1}^n\big<A_1-A_2,X_i\big>^2-\mathbb{E}\big<A_1-A_2,X\big>^2\Big|\\
\leq&\beta_n\big(\|A_1+A_2\|_{L_2(\Pi)}/2\big)+\beta_n\big(\|A_1-A_2\|_{L_2(\Pi)}/2\big),
\end{split}
\end{equation*}
where the last inequality holds because $\frac{A_1\pm A_2}{2}\in\mathcal{A}\big(\|A_1\pm A_2\|_{L_2(\Pi)}/2\big)$. Observe that $\frac{\|A_1\pm A_2\|_{L_2(\Pi)}}{2}\leq \frac{\Delta_1+\Delta_2}{2}$
for all $A_1\in\mathcal{A}(\Delta_1)$ and $A_2\in\mathcal{A}(\Delta_2)$. Therefore,
\begin{equation*}
 \alpha_n(\Delta_1,\Delta_2)\leq 2\beta_n\Big(\frac{\Delta_1+\Delta_2}{2}\Big).
\end{equation*}
It suffices to prove an upper bound for $\beta_n(\Delta)$ for $\Delta\in[\delta^-,\delta^+]$. We need to point out that the upper bound for $\beta_n(\Delta)$ has been claimed in our previous paper \cite{koltchinskii15optimal}
without proof. Since Lemma~\ref{f2lemma} is used for several times in this paper, we give a simple proof based on Dudley's entropy bound and the $L_{\infty}(\Pi_n)$ complexity of unit ball in $\mathbb{H}_m$ equipped with Schatten $1$-norm.

Assume that $\Delta\in[\delta^-,\delta^+]$, the main strategy is to derive the upper bound of $\beta_n(\Delta)$ for $\Delta\in[\delta_j,\delta_{j+1}]$ with $\delta_j=2^j\delta^-$ for $j=0,1,2,\ldots,\floor{\log_2\frac{\delta^+}{\delta^-}}$.
Following a standard argument, the bounds are then taken uniformly over the whole range $[\delta^-,\delta^+]$

For a fixed $j$ such that $\Delta\in[\delta_j,\delta_{j+1}]$, we apply Bousquet's version (see \cite[Chapter~2]{koltchinskii2011oracle}) of Talagrand's inequality for empirical processes and get that with probability at least $1-e^{-t}$,
\begin{equation*}
 \beta_n(\Delta)\leq 2\mathbb{E}\beta_n(\Delta)+2\Delta\sqrt{\frac{t}{nm}}+2\frac{t}{nm}
\end{equation*}
for any $t\geq 1$. We used the facts
\begin{equation*}
 \underset{A\in\mathcal{A}(\Delta)}{\sup}\mathbb{E}\bigl<A,X\bigr>^4\leq \frac{1}{m}\underset{A\in\mathcal{A}(\Delta)}{\sup}\mathbb{E}\bigl<A,X\bigr>^2\leq \frac{\Delta^2}{m}
\end{equation*}
and $\bigl<A,X\bigr>^2\leq \frac{1}{m}$. To control $\mathbb{E}\beta_n(\Delta)$, by the symmetrization inequality,
\begin{equation*}
 \mathbb{E}\beta_n(\Delta)\leq 2\mathbb{E}_X\mathbb{E}_{\epsilon}\underset{A\in\mathcal{A}(\Delta)}{\sup}\Big|\frac{1}{n}\sum_{i=1}^n\epsilon_i\bigl<A,X_i\bigr>^2\Big|
\end{equation*}
where $\epsilon_1,\ldots,\epsilon_n$ are $i.i.d.$ Rademacher random variables.

Fix $X_1,X_2,\ldots,X_n$ and consider the sub-Gaussian process indexed by $A\in\mathcal{A}(\Delta)$ defined as
\begin{equation*}
 G_A:=\frac{1}{\sqrt{n}}\sum_{i=1}^n\epsilon_i\bigl<A,X_i\bigr>^2.
\end{equation*}
This is a sub-Gaussian process with respect to the pseudo-distance 
\begin{equation*}
\begin{split}
d(A_1,A_2):=\mathbb{E}^{1/2}&(G_{A_1}-G_{A_2})^2=\Bigl(\frac{1}{n}\sum_{i=1}^n\bigl<A_1-A_2,X_i\bigr>^2\bigl<A_1+A_2,X_i\bigr>^2\Bigr)^{1/2}\\
\leq&2\sigma_n\|A_1-A_2\|_{L_{\infty}(\Pi_n)},
\end{split}
\end{equation*}
where $\sigma_n^2:=\underset{A\in\mathcal{A}(\Delta)}{\sup}\frac{1}{n}\sum_{i=1}^n\bigl<A,X_i\bigr>^2$. According to Dudley's entropy bound,
\begin{equation*}
 \mathbb{E}_{\epsilon}\underset{A\in\mathcal{A}(\Delta)}{\sup} |G_A|\lesssim \int_0^{4\sigma_n/\sqrt{m}}H^{1/2}(\mathcal{A}(\Delta),d,u)du,
\end{equation*}
where the entropy number $H(\mathcal{A}(\Delta),d,u)=\log N(\mathcal{A}(\Delta),d,u)$, the logarithmic of $u$-covering number of $\mathcal{A}(\Delta)$ with respect to the pseudo-metric $d$.

Since $d(A_1,A_2)\leq 2\sigma_n\|A_1-A_2\|_{L_{\infty}(\Pi_n)}$,
$$
H^{1/2}(\mathcal{A}(\Delta),d,u)\leq H^{1/2}(\mathcal{A}(\Delta),L_{\infty}(\Pi_n),\frac{u}{2\sigma_n}).
$$
As a consequence,
\begin{equation*}
\begin{split}
 \mathbb{E}_{\epsilon}\underset{A\in\mathcal{A}(\Delta)}{\sup} G_A\lesssim \int_0^{4\sigma_n/\sqrt{m}}H^{1/2}&(\mathcal{A}(\Delta),L_{\infty}(\Pi_n),\frac{u}{2\sigma_n})du\\
 &\leq 2\sigma_n\int_0^{2/\sqrt{m}}H^{1/2}(\mathcal{A}(\Delta),L_{\infty}(\Pi_n),u)du.
 \end{split}
\end{equation*}

The $L_{\infty}(\Pi_n)$-complexity of unit balls in $\mathbb{H}_m$ equipped with nuclear norm distance has been thoroughly studied. When $X_1,\ldots,X_n$ are fixed, the vector $\bigl(\bigl<A,X_1\bigr>,\ldots,\bigl<A,X_n\bigr>\bigr)'$ belongs to the cube $[-1/\sqrt{m},1/\sqrt{m}]^n$. 
The $l_{\infty}$-covering number is upper bounded by
\begin{equation*}
 N(\mathcal{A}(\Delta),L_{\infty}(\Pi_n),u)\leq \Bigl(1+\frac{2}{u\sqrt{m}}\Bigr)^n.
\end{equation*}
This bound will be used when $u$ is small. When $u$ is large, we apply the following bound, see \cite[(21)]{liu2011universal}, \cite{guedon2008},\cite[Lemma A5]{aubrun2009},
\begin{equation*}
 N(\mathcal{A}(\Delta),L_{\infty}(\Pi_n),u)\leq \exp\Bigl\{C\frac{\log^3m\log n}{u^2m}\Bigr\}
\end{equation*}
for some constant $C>0$. Then, by setting $K=\frac{1}{\sqrt{nm}}$,
\begin{equation*}
 \begin{split}
  \mathbb{E}_{\epsilon}\underset{A\in\mathcal{A}(\Delta)}{\sup} G_A\lesssim&\sigma_n\int_0^{K}\sqrt{n}\log^{\frac{1}{2}}\Bigl(1+\frac{2}{u\sqrt{m}}\Bigr)du+\sigma_n\int_K^{2/\sqrt{m}}\frac{\log^{3/2}m\log^{1/2}n}{u\sqrt{m}}du\\
\lesssim& \sigma_n\sqrt{n}K\log(1+\frac{2}{K\sqrt{m}})+\frac{\sigma_n}{\sqrt{m}}\log^{3/2}m\log^{1/2}n\log\frac{1}{K\sqrt{m}}\\
\lesssim& \frac{\sigma_n}{\sqrt{m}} \log^{3/2}m\log^{3/2}n.
 \end{split}
\end{equation*}
Therefore, we conclude that
\begin{equation*}
 \mathbb{E}\beta_n(\Delta)=\frac{1}{\sqrt{n}}\mathbb{E}_X\mathbb{E}_{\epsilon}\underset{A\in\mathcal{A}(\Delta)}{\sup}G_A\lesssim\frac{1}{\sqrt{nm}} \mathbb{E}_X\sigma_n\log^{3/2}m\log^{3/2}n.
\end{equation*}
Note that
\begin{equation*}
\begin{split}
 \mathbb{E}_X\sigma_n=\mathbb{E}_X&\sqrt{\underset{A\in\mathcal{A}(\Delta)}{\sup}\frac{1}{n}\sum_{i=1}^n\bigl<A,X_i\bigr>^2}\leq \sqrt{\mathbb{E}_X\underset{A\in\mathcal{A}(\Delta)}{\sup}\frac{1}{n}\sum_{i=1}^n\bigl<A,X_i\bigr>^2}\\
\leq&\sqrt{\mathbb{E}\beta_n(\Delta)+\Delta^2}.
\end{split}
\end{equation*}
Therefore, we get
\begin{equation*}
 \mathbb{E}\beta_n(\Delta)\lesssim \sqrt{\mathbb{E}\beta_n(\Delta)+\Delta^2}\frac{\log^{3/2}m\log^{3/2}n}{\sqrt{nm}},
\end{equation*}
which can be simplified into
\begin{equation*}
 \mathbb{E}\beta_n(\Delta)\lesssim \Delta \frac{\log^{3/2}m\log^{3/2}n}{\sqrt{nm}}+\frac{\log^3m\log^3n}{nm}.
\end{equation*}
Therefore, for $\Delta\in[\delta_j,\delta_{j+1}]$, with probability at least $1-e^{-t}$,
\begin{equation*}
 \beta_n(\Delta)\leq C\Delta \frac{\log^{3/2}m\log^{3/2}n}{\sqrt{nm}}+C\frac{\log^3m\log^3n}{nm}+2\Delta\sqrt{\frac{t}{nm}}+2\frac{t}{nm}.
\end{equation*}
for some $C>0$. By making it uniform over all $j=0,1,\ldots,\floor{\log_2\frac{\delta^+}{\delta^-}}$ and adjusting $t$ to $t+\log(\log_2\frac{\delta^+}{\delta^-}+2)$, we get the uniform upper bound of $\beta_n(\Delta)$ for $\Delta\in[\delta^-,\delta^+]$.
\end{proof}

\bibliographystyle{abbrv}
\bibliography{refer}

\begin{thebibliography}{10}

\bibitem{alquier2013rank}
P.~Alquier, C.~Butucea, M.~Hebiri, K.~Meziani, and T.~Morimae.
\newblock Rank-penalized estimation of a quantum system.
\newblock {\em Physical Review A}, 88(3):032113, 2013.

\bibitem{aubrun2009}
G.~Aubrun.
\newblock On almost randomizing channels with a short {Kraus} decomposition.
\newblock {\em Communications in Mathematical Physics}, 288(3):1103--1116,
  2009.

\bibitem{caioptimal}
T.~Cai, D.~Kim, Y.~Wang, M.~Yuan, and H.~H. Zhou.
\newblock Optimal large-scale quantum state tomography with {Pauli}
  measurements.
\newblock {\em The Annals of Statistics}, 44(2):682--712, 2016.

\bibitem{candes2011tight}
E.~J. Cand{\'e}s and Y.~Plan.
\newblock Tight oracle inequalities for low-rank matrix recovery from a minimal
  number of noisy random measurements.
\newblock {\em IEEE Transactions on Information Theory}, 57(4):2342--2359,
  2011.

\bibitem{flammia2012quantum}
S.~T. Flammia, D.~Gross, Y.-K. Liu, and J.~Eisert.
\newblock Quantum tomography via compressed sensing: error bounds, sample
  complexity and efficient estimators.
\newblock {\em New Journal of Physics}, 14(9):095022, 2012.

\bibitem{gross2011recovering}
D.~Gross.
\newblock Recovering low-rank matrices from few coefficients in any basis.
\newblock {\em IEEE Transactions on Information Theory}, 57(3):1548--1566,
  2011.

\bibitem{gross2010quantum}
D.~Gross, Y.-K. Liu, S.~T. Flammia, S.~Becker, and J.~Eisert.
\newblock Quantum state tomography via compressed sensing.
\newblock {\em Physical Review Letters}, 105(15):150401, 2010.

\bibitem{guedon2008}
O.~Gu{\'e}don, S.~Mendelson, A.~Pajor, and N.~Tomczak-Jaegermann.
\newblock Majorizing measures and proportional subsets of bounded orthonormal
  systems.
\newblock {\em Revista Matem{\'a}tica Iberoamericana}, 24(3):1075--1095, 2008.

\bibitem{koltchinskii2009dantzig}
V.~Koltchinskii.
\newblock The {D}antzig selector and sparsity oracle inequalities.
\newblock {\em Bernoulli}, 15(3):799--828, 2009.

\bibitem{koltchinskii2011oracle}
V.~Koltchinskii.
\newblock {\em Oracle Inequalities in Empirical Risk Minimization and Sparse
  Recovery Problems: {\'E}cole d'{\'E}t{\'e} de Probabilit{\'e}s de Saint-Flour
  XXXVIII-2008}.
\newblock Springer, 2011.

\bibitem{koltchinskii2011neumann}
V.~Koltchinskii.
\newblock von {Neumann} entropy penalization and low-rank matrix estimation.
\newblock {\em The Annals of Statistics}, 39(6):2936--2973, 2011.

\bibitem{koltchinskii2013sharp}
V.~Koltchinskii.
\newblock Sharp oracle inequalities in low rank estimation.
\newblock In {\em Empirical Inference}, pages 217--230. Springer, 2013.

\bibitem{koltchinskii2011nuclear}
V.~Koltchinskii, K.~Lounici, and A.~B. Tsybakov.
\newblock Nuclear-norm penalization and optimal rates for noisy low-rank matrix
  completion.
\newblock {\em The Annals of Statistics}, 39(5):2302--2329, 2011.

\bibitem{koltchinskii15optimal}
V.~Koltchinskii and D.~Xia.
\newblock Optimal estimation of low rank density matrices.
\newblock {\em Journal of Machine Learning Research}, 16:1757--1792, 2015.

\bibitem{liu2011universal}
Y.-K. Liu.
\newblock Universal low-rank matrix recovery from {Pauli} measurements.
\newblock In {\em Advances in Neural Information Processing Systems}, pages
  1638--1646, 2011.

\bibitem{mendelson2014upper}
S.~Mendelson.
\newblock Upper bounds on product and multiplier empirical processes.
\newblock {\em Stochastic Processes and their Applications}, 2016.

\bibitem{negahban2012restricted}
S.~Negahban and M.~J. Wainwright.
\newblock Restricted strong convexity and weighted matrix completion: Optimal
  bounds with noise.
\newblock {\em Journal of Machine Learning Research}, 13(May):1665--1697, 2012.

\bibitem{Nielsen2000}
M.~Nielsen and I.~Chuang.
\newblock {\em Quantum Computation and Quantum Information.}
\newblock Cambridge University Press, 2000.

\bibitem{pajor1998metric}
A.~Pajor.
\newblock Metric entropy of the {Grassmann} manifold.
\newblock {\em Convex Geometric Analysis}, 34:181--188, 1998.

\bibitem{rohde2011estimation}
A.~Rohde and A.~B. Tsybakov.
\newblock Estimation of high-dimensional low-rank matrices.
\newblock {\em The Annals of Statistics}, 39(2):887--930, 2011.

\bibitem{tropp2012user}
J.~A. Tropp.
\newblock User-friendly tail bounds for sums of random matrices.
\newblock {\em Foundations of Computational Mathematics}, 12(4):389--434, 2012.

\bibitem{intro}
A.~B. Tsybakov.
\newblock {\em Introduction to Nonparametric Estimation}.
\newblock Springer, 2008.

\bibitem{wang2013asymptotic}
Y.~Wang.
\newblock Asymptotic equivalence of quantum state tomography and noisy matrix
  completion.
\newblock {\em The Annals of Statistics}, 41(5):2462--2504, 2013.

\bibitem{xia2014optimal}
D.~Xia.
\newblock Optimal schatten-q and ky-fan-k norm rate of low rank matrix
  estimation.
\newblock {\em arXiv preprint arXiv:1403.6499}, 2014.

\bibitem{Xia2016Estimation}
D.~Xia and V.~Koltchinskii.
\newblock Estimation of low rank density matrices: Bounds in schatten norms and
  other distances.
\newblock {\em Electron. J. Statist.}, 10(2):2717--2745, 2016.

\end{thebibliography}

\end{document}